\documentclass[twoside]{article}

\usepackage[accepted]{aistats2025}
\usepackage{enumitem}
\usepackage{pifont}
\usepackage{threeparttable}
\usepackage{multirow}
\usepackage{fdsymbol}
\usepackage{caption}
\usepackage{subcaption}
\usepackage{graphicx}
\usepackage{wrapfig}
\usepackage{fdsymbol}
\usepackage{enumitem}
\usepackage{color}
\usepackage{caption}
\usepackage{subcaption}
\usepackage{hyperref}
\usepackage{url}
\usepackage{soul}
\usepackage{amsthm}
\usepackage{pifont}
\usepackage{array}
\usepackage{algorithm}
\usepackage{algorithmic}
\usepackage{booktabs}

\newtheorem{Theorem}{Theorem}[section]
\newtheorem{Definition}{Definition}[section]
\newtheorem{Proposition}{Proposition}[section]
\newtheorem{Assumption}{Assumption}[section]

\theoremstyle{remark}

\usepackage[round]{natbib}

\begin{document}

\renewcommand{\thefootnote}{\fnsymbol{footnote}}
\runningauthor{J. Zhang, R. Ding, Q. Fu, B. Huang, Z. Deng, Y. Hua, H. Guan, S. Han, D. Zhang}

\twocolumn[

\aistatstitle{Learning Identifiable Structures Helps Avoid Bias in DNN-based Supervised Causal Learning}

\aistatsauthor{ Jiaru Zhang\footnotemark{} \And Rui Ding\footnotemark{} \And Qiang Fu  }
\aistatsaddress{ Shanghai Jiao Tong University \\\texttt{jiaruzhang@sjtu.edu.cn} \And  Microsoft \\\texttt{juding@microsoft.com} \And  Microsoft\\ \texttt{qifu@microsoft.com} } 
\aistatsauthor{Huang Bojun \And Zizhen Deng\And   Yang Hua }
\aistatsaddress{Sony Research \\ \texttt{bojhuang@gmail.com} \And Peking University \\ \texttt{dengzizhen557@outlook.com}  \And Queen’s University Belfast \\ \texttt{y.hua@qub.ac.uk}} 
\aistatsauthor{ Haibing Guan \And Shi Han \And Dongmei Zhang }
\aistatsaddress{Shanghai Jiao Tong University \\ \texttt{hbguan@sjtu.edu.cn}\And Microsoft\\ \texttt{shihan@microsoft.com} \And Microsoft \\ \texttt{dongmeiz@microsoft.com}} 
]

\footnotetext[1]{The work was done during his internship at Microsoft Research Asia.}
\footnotetext[2]{Corresponding author.}

\renewcommand{\thefootnote}{\arabic{footnote}}
\setcounter{footnote}{0}
\begin{abstract}
Causal discovery is a structured prediction task that aims to predict causal relations among variables based on their data samples.
Supervised Causal Learning (SCL) is an emerging paradigm in this field.
Existing Deep Neural Network (DNN)-based methods commonly adopt the “Node-Edge approach”,
in which the model first computes an embedding vector for each variable-node, then uses these variable-wise representations to concurrently and independently predict for each directed causal-edge.
In this paper, we first show that this architecture has some systematic bias that cannot be mitigated regardless of model size and data size. 
We then propose SiCL, a DNN-based SCL method that predicts a skeleton matrix together with a v-tensor (a third-order tensor representing the v-structures). According to the Markov Equivalence Class (MEC) theory, both the skeleton and the v-structures are \emph{identifiable} causal structures under the canonical MEC setting, so predictions about skeleton and v-structures do not suffer from the identifiability limit in causal discovery, thus SiCL can avoid the systematic bias in Node-Edge architecture, and enable consistent estimators for causal discovery. Moreover, SiCL is also equipped with a specially designed pairwise encoder module with a unidirectional attention layer to model both internal and external relationships of pairs of nodes. Experimental results on both synthetic and real-world benchmarks show that SiCL significantly outperforms other DNN-based SCL approaches.
\end{abstract}

\section{Introduction} \label{sec:int}

Causal discovery seeks to infer causal structures from an observational data sample.
Supervised Causal Learning (SCL) \citep{dai2023ml4c,ke2023learning,ma2022ml4s} is an emerging paradigm in this field. The basic idea is to consider causal discovery as a \emph{structured prediction} task, and to train a prediction model using supervised learning techniques. 
At training time, a training dataset comprising a variety of causal mechanisms and their associated data samples is generated. 
The prediction model is then trained to take such a data sample as input, and to output predictions about the causal mechanism behind the data sample.
Compared to traditional rule-based or unsupervised 
methods~\citep{glymour2019review}, the SCL method has demonstrated strong empirical performance~\citep{dai2023ml4c,ma2022ml4s}, as well as robustness against sample size and distribution shift \citep{ke2023learning,lorchamortized}.

Deep Neural Network (DNN)-based SCL employs DNN as the prediction model. It allows end-to-end training, removing the need for manual feature engineering. Additionally, it can handle both continuous and discrete data types effectively, and can learn latent representations.
A specific DNN architecture, first introduced by \cite{lorchamortized}, is particularly popular in recent DNN-based SCL works. The model first transforms the given data sample into a set of node-wise feature vectors, each representing an individual variable (corresponding to a node in the associated causal graph).
Based on these node-wise features, the model then outputs a weighted adjacency matrix $A$, where $A_{ij}\in[0,1]$  is an estimated probability for the directed edge $i \rightarrow j$ (meaning that $i$ is a direct cause of $j$). 
Finally, the adjacency matrix of an inferred causal graph $G$ is obtained as a Bernoulli sample of $A$, where each entry $G_{ij} \in \{0,1\}$ is sampled \emph{independently}, following probability $A_{ij}$. 
\textcolor{black}{For convenience, we call} such a model architecture as the ``Node-Edge'' \textcolor{black}{architecture}, as the representation is learned for individual nodes and the probability is estimated and sampled for individual directed edges. 

Despite its popularity and encouraging results~\citep{lorchamortized,Zhu2020Causal,dpdag,varamballydiscovering}, we identify two limitations for the Node-Edge approach:

\textcolor{black}{First, the Node-Edge architecture imposes a fundamental bias in the inferred causal relations. Specifically, given an observational data sample $D$, the existence of a directed causal edge $i \rightarrow j$ may \emph{necessarily} depend on the existence of other edges. But the existing Node-Edge models predict each edge separately and independently, so the probability prediction $A_{ij}$ made by such models is only conditioned on the input sample $D$, not on the sampling result of other entries of $A$,  thereby failing to capture the crucial inter-edge dependency in its probability estimation.
}

As a simple example, a Node-Edge model maintaining the possibility of both $G_1: X\rightarrow T \rightarrow Y$ and $G_2: X\leftarrow T \leftarrow Y$ would necessarily have a non-zero probability to output the edges $X\rightarrow T$ and $T \leftarrow Y$, thus cannot rule out the possibility of $G_3: X\rightarrow T \leftarrow Y$, even though $G_3$ is impossible to be the groundtruth causal graph behind a data sample $D$ compatible with $G_1$ and $G_2$~\citep{verma1990equivalence}. 
Crucially, there is no way to tell $G_1$ from $G_2$ based on observational data in general cases~\citep{andersson1997characterization,meek1995strong}. 
It means that for any Node-Edge model to be sound, it has to maintain the possibility of both $G_1$ and $G_2$ (when observing a data sample compatible with any of them), leading to an inevitable error probability to output the impossible graph $G_3$ on the other hand.

\textcolor{black}{Second, the Node-Edge architecture does not explicitly represent the features about node pairs, which we argue are essential for observational causal discovery.}
For example, a causal edge $X\rightarrow Y$ can exist only if the node pair $\langle X, Y\rangle$ demonstrates \textit{persistent dependency} \citep{ma2022ml4s,spirtes2000causation}, meaning that $X$ and $Y$ remain statistically dependent regardless of conditioning on any subset of other variables. As another example, for causal DAGs, a sufficient condition to determine the causal direction between a persistently dependent node pair $\langle X, Y\rangle$ is that $X$ and $Y$ exhibits \emph{orientation asymmetry}, meaning that there exists a third variable $Z$ such that $X$ is persistently dependent to $Z$ but $Y$ can become independent to $Z$ conditioned on a variable-set $\mathbf{S}\not\ni X$ (or vice versa). A feature like persistent dependency or orientation asymmetry is, in its nature, a collective property of a node pair, but not of any individual node alone. 

To address these limitations, in this paper, we propose a novel DNN-based SCL approach, called Supervised Identifiable Causal Learning (SiCL). 
\textcolor{black}{The neural network in SiCL does not seek to predict the probabilities of directed edges, but tries to predict a skeleton matrix together with v-tensor, a third-order tensor representing the v-structures.}
According to the Markov Equivalence Class (MEC) theory of causal discovery, skeleton and v-structures are \emph{identifiable} causal structures under the canonical MEC setting (while the directed edges are not), so predictions about skeleton and v-structures do not suffer from the (non-)identifiability limit. 
By leveraging this insight, our theory-inspired DNN architecture completely avoids the systematic bias in edge-prediction models as previously discussed, and enables \emph{consistent} neural-estimators\footnote{Recall that a statistical estimator is \emph{consistent} if it converges to the groundtruth given infinite data.} for causal discovery. 
Moreover, SiCL is also equipped with a specially designed pairwise encoder module with a unidirectional attention layer.
With both node features and node-pair features as the layer input, it can model both internal and external relationships of pairs of nodes.
Experimental results on both synthetic and real-word benchmarks show that \textcolor{black}{SiCL} can effectively address the two \textcolor{black}{above-mentioned limitations}, and the resulted SiCL solution significantly outperforms other DNN-based SCL approaches with more than 50\% performance improvement in terms of SHD (Structural Hamming Distance) on the real-world Sachs data. The codes are publicly available at \url{https://github.com/microsoft/reliableAI/tree/main/causal-kit/SiCL}.

\section{Background and Related Work}\label{sec:bg}
A Causal Graphical Model is defined by a joint probability distribution $P$ over multiple random variables and a DAG $G$. Each node $X_i$ in $G$ represents a variable in $P$, and a directed edge $X_i \rightarrow X_j$ represents a direct cause-effect relation from $X_i$ to $X_j$.
A causal discovery task generally asks to infer about $G$ from an i.i.d. sample of $P$.

However, there is a well-known identifiability limit for causal discovery.
In general, the causal DAG is only identifiable up to an equivalence class. 
Studies of this identifiability limit under a canonical assumption setting have led to the well-established MEC theory \citep{frydenberg1990chain,verma1990equivalence}.
We call a causal feature, \textit{MEC-identifiable}, if the value of this feature is invariant among the equivalence class under the canonical MEC assumption setting. 
It is known that such MEC-identifiable features include the skeleton and the set of v-structures, which we briefly present in the following.

A \textit{skeleton} $E$ defined over the data distribution $P$ is an undirected graph where an edge exists between $X_i$ and $X_j$ if and only if $X_i$ and $X_j$ are always dependent in $P$, i.e., $\forall Z \subseteq\left\{X_1, X_2, \cdots, X_d\right\} \backslash \left\{X_i, X_j \right\}$, we have $X_i \nperp X_j | Z$.
Under mild assumptions (such as that $P$ is Markovian and faithful to the DAG $G$; see details in Appendix Sec. \ref{sec:da}), 
the skeleton is the same as the corresponding undirected graph of the DAG $G$ \citep{spirtes2000causation}. 
A triple of variables $\langle X, T, Y \rangle$ is an \textit{Unshielded Triple (UT)} if $X$ and $Y$ are both adjacent to $T$ but not adjacent to each other in (the skeleton of) $G$.
It becomes a \textit{v-structure} denoted as $X \rightarrow T \leftarrow Y$ if the directions of the edges are from $X$ and $Y$ to $T$.

Two graphs are Markov equivalent if and only if they have the same skeleton and v-structures. 
The \textit{Markov equivalence class (MEC)} can be represented by a \textit{Completed Partially Directed Acyclic Graph (CPDAG)} consisting of both directed and undirected edges. We use $CPDAG(G)$ to denote the CPDAG derived from $G$.
According to the theorem of Markov completeness \citep{meek1995strong}, 
we can only identify a causal graph up to its MEC, i.e., the CPDAG, unless additional assumptions are made (see the remark below).
This means that each (un)directed edge in $CPDAG(G)$ indicates a (non)identifiable causal relation.

\textbf{Remark:} The MEC-based identifiability theory is applicable in the general-case setting, when we take into account all possible distributions $P$. 
It is known this identifiability limit could be broken (i.e., an undirected edge in the CPDAG could be oriented) \textit{if} we assume that the data follows some special class of distributions, e.g., linear non-Gaussians, additive noise models, post-nonlinear or location-scale models~\citep{peters2014causal,shimizu2011directlingam,zhang2009identifiability,immer2023identifiability}. 
These assumptions are sometimes hard to verify in practice, so this paper considers the general-case setting.
More discussions on the identifiability and causal assumptions are provided in Appendix Sec. \ref{sec:dica}.

\subsection{Related Work}
In traditional methods of causal discovery, constraint-based methods are mostly related to our work.
They aim to identify the DAG that is consistent with inter-variable conditional independence constraints. 
These methods first identify the skeleton and then conduct orientation based on v-structure identification \citep{yu2016review}. 
The output is a CPDAG which represents the MEC.
Notable algorithms in this category include PC \citep{spirtes2000causation}, along with variations such as Conservative-PC \citep{ramsey2012adjacency}, PC-stable \citep{colombo2014order}, and Parallel-PC \citep{le2016fast}. 
Compared to constraint-based methods, both our approach and theirs are founded upon the principles of MEC theory for estimating skeleton and v-structures. However, whereas traditional methods rely on symbolic reasoning based on explicit constraints, we employ DNNs to capture the essential causal information intricately linked with these constraints.

Score-based methods aim to find an optimal DAG according to a predefined score function, subject to combinatorial constraints. 
These methods employ specific optimization procedures such as forward-backward search GES \citep{chickering2002optimal}, hill-climbing \citep{koller2009probabilistic}, and integer programming \citep{cussens2011bayesian}.
Continuous optimization methods transform the discrete search procedure into a continuous equality constraint.
NOTEARS \citep{zheng2018dags} formulates the acyclic constraint as a continuous equality constraint and is further extended by DAG-GNN \citep{yu2019dag}, DECI \citep{geffner2022deep} to support non-linear causal relations. 
DECI \citep{geffner2022deep} is a flow-based model which can perform both causal discovery and inference on non-linear additive noise data.
Recently, ENCO \citep{lippe2021efficient} is proposed as a continuous optimization method where the edge orientation is modeled as a separate parameter to maintain the acyclicity.
It is guaranteed to converge to the correct graph if interventions on all variables are available.
 \textcolor{black}{RL-BIC \citep{Zhu2020Causal} utilizes Reinforcement Learning to search for the optimal DAG.}
These methods can be viewed as unsupervised since they do not access additional datasets associated with ground truth causal relations.
We refer to \cite{glymour2019review,vowels2022d} for a thorough exploration of this literature.

SCL begins from orienting edges in the bivariate cases under the functional causal model formalism. 
Methods such as RCC \citep{lopez2015randomized} and NCC \citep{lopez2017discovering} have outperformed unsupervised approaches like ANM \citep{hoyer2008nonlinear} or IGCI \citep{janzing2012information}.
For multivariate cases, ML4S \citep{ma2022ml4s} proposes a supervised approach specifically for skeleton learning. 
Complementary to ML4S, ML4C \citep{dai2023ml4c} takes both data and skeleton as input and classifies unshielded triples as either v-structures or non-v-structures. 
\cite{petersen2023causal} proposes a SLdisco method, utilizing SCL approach to address some limitations of PC and GES.

DNN-based SCL has emerged as a prominent approach for enabling end-to-end causal learning. 
Two notable works in this line, namely AVICI \citep{lorchamortized} and CSIvA \citep{ke2023learning}, introduced an alternating attention mechanism to enable permutation invariance across samples and variables. Both methods learn individual representation for each node, which is then used to predict directed edges. Among them, AVICI considers the task of predicting DAG from observational data and adopts exactly the Node-Edge architecture, hence suffers from the issues as discussed in Sec. \ref{sec:int}. On the other hand, CSIvA requires additional interventional data as input to identify the full DAG, and applies an autoregressive DNN architecture where edges are predicted sequentially by multiple inference runs. Therefore, this autoregressive approach incurs very high inference cost due to the quadratic number of model runs required (w.r.t. the number of variables in question), as we experimentally verify in Appendix Sec. \ref{sec:auto}. In contrast, the method proposed in this paper only requires a single run of the DNN model. Besides that, our method also differs from both AVICI and CSIvA in terms of the usage of pairwise embedding vectors.

\section{Limitations of the Node-Edge \textcolor{black}{Architecture}} \label{sec:met:lim}



\textcolor{black}{The Node-Edge architecture is common and has been adopted to generate the output DAG $G$ in the literature \citep{lorchamortized,Zhu2020Causal,dpdag,varamballydiscovering}.}
In this \textcolor{black}{architecture}, each entry $G_{ij}$ in the DAG is independently sampled from $A_{ij}$, an entry in the adjacency matrix $A$. This entry $A_{ij}$ represents the probability that $i$ directly causes $j$. 
We introduce a simple yet effective example setting with only three variables $X$, $Y$, and $T$ to reveal its limitation.

\textcolor{black}{Considering a simulator that generates DAGs with equal probability from two causal models: }In model 1, the causal graph is $G_1: X \rightarrow T \rightarrow Y$, and the variables follow $X \sim \mathcal{N} (0, 1)$, $T = X + \mathcal{N}(0, 1)$, $Y = T + \mathcal{N}(0, 1)$.
In model 2, the causal graph is $G_2: X \leftarrow T \leftarrow Y$, and the variables follow $ Y = \mathcal{N}(0, 3)$, $T = \frac{2}{3}Y + \mathcal{N}(0, \frac{2}{3})$, $X = 0.5T + \mathcal{N}(0, 0.5)$.
In this case, data samples coming from both causal models follow the same joint distribution, which makes $G_1$ and $G_2$ inherently indistinguishable (from observational data sample).

More importantly, when the fully-directed causal DAGs are used as the learning target (as the Node-Edge approach does), an optimally trained neural network will predict $0.5$ probabilities on the directions of the two edges $X - T$ and $T - Y$.
As a result, with $0.25$ probability the graph sampling outcome would be $X \rightarrow T \leftarrow Y$ (see Fig. \ref{fig:ps} in the Appendix).
This error probability is rooted from the fact that the Bernoulli sampling of the edge $X \rightarrow T$ is not conditioned on the sampling result of the edge $T \leftarrow Y$. Consequently, it is a bias that cannot be avoided even if the DNN has perfectly modeled the \emph{marginal probability} of each edge (marginalized over other edges) given input data.

\color{black}
We further find that $0.25$ is not the worst-case error rate yet. 
Formally, for a distribution $Q$ over a set of graphs, we define the graph distribution where the edges are independently sampled from the marginal distribution as $M(Q)$, i.e., for any causal edges $e_1$ and $e_2$, $P_{G\sim Q}(e_1 \in G) = P_{G\sim M(Q)}(e_1 \in G) = P_{G\sim M(Q)}(e_1 \in G | e_2 \in G)$.
In general, a Node-Edge model optimally trained on data samples $D$ coming from the distribution $Q$ will essentially learn to predict $M(Q)$ (when given the same data samples $D$ at test time).
The following proposition shows that for causal graphs with star-shaped skeleton, with a chance of $26.42\%$ the graph sampled from the marginal distribution $M(Q)$ would be incorrect.

\begin{Proposition}
Let $\mathcal{G}_n$ be the set of graphs with $n+1$ nodes where there is a central node $y$ such that (1) every other node is connected to $y$, (2) there is no edge between the other nodes, \textcolor{black}{and} (3) there is at most one edge pointing to $y$. 
We have 
\begin{align}
\sup_n \max_{Q} P_{G \sim M(Q)}(G \nin \mathcal{G}_n) = 
1 - \frac{2}{e} \approx 0.2642.
\end{align}
\label{prop:star}
\end{Proposition}
The proof is provided in Appendix Sec. \ref{sec:mgcs}. 
It indicates that an edge-predicting neural network could suffer from an inevitable error rate of $0.2642$ even if it is perfectly trained.
\color{black}
In contrast, models that predict skeleton and v-structures would have a theoretical asymptotic guarantee of the consistency under canonical assumption. 
The details, proof and relevant discussions are provided in Appendix Sec. \ref{sec:app:tg}.

\section{The SiCL Method} \label{sec:algo}
In light of the limitations as discussed, we propose a new DNN-based SCL method in this section, named \textbf{SiCL} (\textbf{S}upervised \textbf{i}dentifiable \textbf{C}ausal \textbf{L}earning).

\subsection{Overall Workflow}
\begin{figure*}[htb]
\centering
\includegraphics[width=\linewidth]{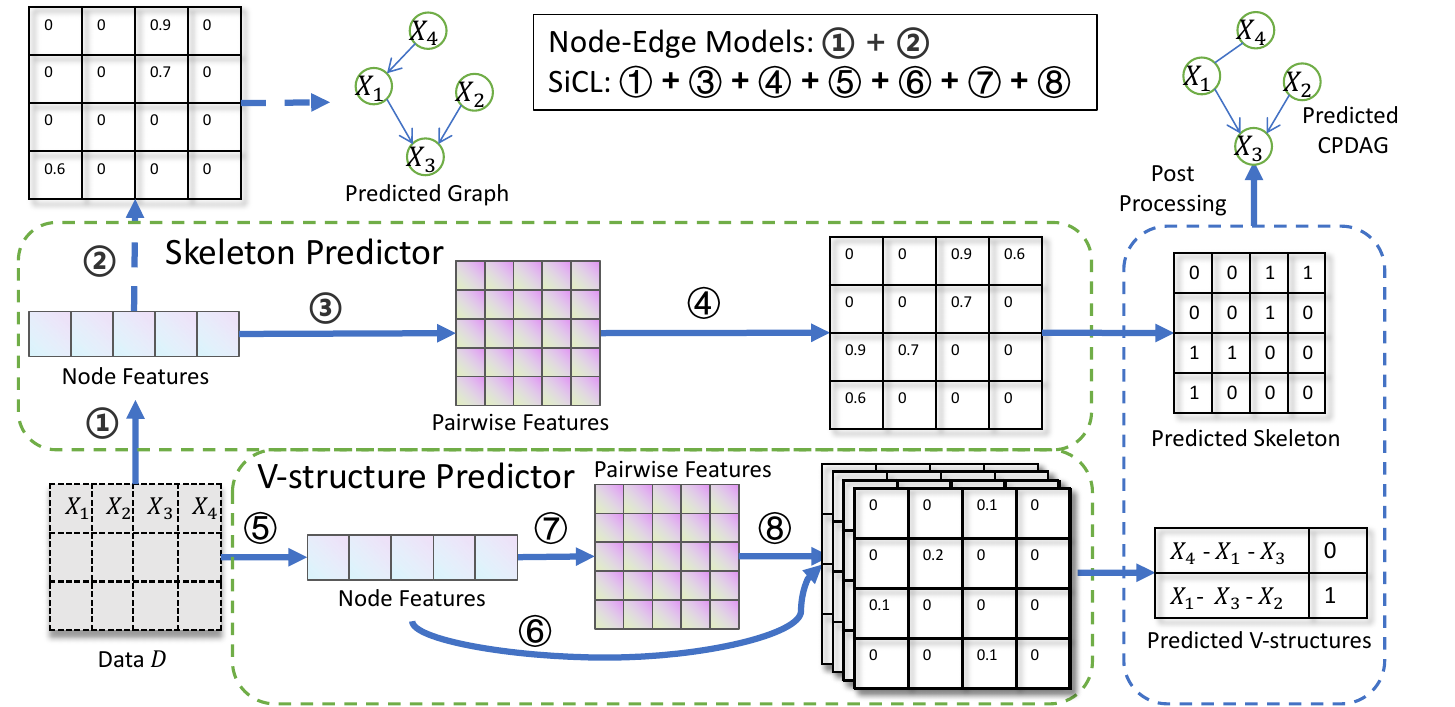}
    \caption{The inference workflow of SiCL.}
    \label{fig:ww}
      
\end{figure*}
Following the standard DNN-based SCL paradigm, the core inference process is implemented as a DNN. The DNN takes a data sample encoded by a matrix $D\in\mathbb{R}^{n \times d}$ as input, with $d$ being the number of observable variables and $n$ being the number of observations. In contrast to previous \textcolor{black}{Node-Edge} approaches, the SiCL method does not use the DNN to directly predict the causal graph, but instead seeks to predict the skeleton and the v-structures of the causal graph, which amount to the \emph{MEC-identifiable} causal structures as mentioned previously. 

Specifically, our DNN outputs two objects: (1) a skeleton prediction matrix $S \in [0,1]^{d \times d}$ where $S_{ij}$ models the conditional probability  
\textcolor{black}{of the existence of the \emph{undirected} edge $X_i - X_j$, conditioned on the input data sample $D$, }
and (2) a v-structure prediction tensor $V \in [0,1]^{d\times d\times d}$ where $V_{ijk}$ models the conditional probability 
\textcolor{black}{of the existence of the v-structure component $X_j \rightarrow X_i \leftarrow X_k$, again conditioned on $D$.}
In our implementation, $S$ and $V$ are generated by two separate sub-networks, called Skeleton Predictor Network (SPN) and V-structure Predictor Network (VPN), respectively. 
To further address the limitation of only having node-wise features for Node-Edge models, we propose to equip SPN and VPN with Pairwise Encoder modules to explicitly capture node-pair-wise features.

Based on the skeleton prediction matrix $S$ and v-structure prediction tensor $V$, we infer the skeleton and v-structures of the causal graph; from the two we can determine a unique CPDAG. The CPDAG encodes a Markov equivalence class, from which we can pick up a graph instance as the prediction of the causal DAG (if needed). 
\textcolor{black}{Figure \ref{fig:ww} provides a diagram of the overall inference workflow of SiCL, and a pseudo-code of the workflow is given by Algorithm \ref{alg:workflow} in Appendix.}

Parameters of the SPN and VPN are trained following the standard supervised learning procedure. In our implementation, we only use synthetic training data, which is relatively easy to obtain, and yet could often lead to strong performance on real-world workloads~\citep{ke2023learning}.

In the following, we elaborate the DNN architecture, the learning targets, as well as the post-processing procedure used in the SiCL method.

\subsection{Feature Extraction} \label{sec:fem}

\textbf{Input Processing and Node Feature Encoder.} Given input data sample, which is a matrix $D\in\mathbb{R}^{n \times d}$, the input processing module contains a linear layer for continuous input data or an embedding layer for discrete input data, yielding the raw node features $\mathcal{F}^{raw}_{il}$ for each node $i$ in each observation $l$.
After that, we employ a node feature encoder to further process the raw node features into the final node features $\mathcal{F}_{il}$.
Similar to previous papers \citep{ke2023learning,lorchamortized}, the node feature encoder is a transformer-like network comprising attention layers over the observation dimension and the node dimension alternately, which naturally maintains permutation equivalence across both variable and data dimension \textcolor{black}{because of the intrinsic symmetry of attention operations}.
\textcolor{black}{More details about the node feature encoder are presented in Appendix Sec. \ref{sec:dnf} due to page limit}.

\textbf{Pairwise Encoder.} Given node features $\mathcal{F} \in \mathbb{R}^{d\times n \times h}$ for all the $d$ nodes, the goal of pairwise encoder is to encode their pairwise relationships by $d^2$ pairwise features, represented as a tensor $\mathcal{P} \in \mathbb{R}^{d\times d \times n \times h}$, where $\mathcal{P}_{ijl} \in \mathbb{R}^{h}$ is a pairwise feature corresponding to the node pair $(i,j)$ in observation $l$.
As argued in Sec. \ref{sec:int}, both ``internal'' information (i.e., the pairwise relationship) and ``external information (e.g., the context of the conditional separation set) of node pairs are needed to capture persistent dependency and orientation asymmetry.
Our pairwise encoder module \textcolor{black}{is designed to model the internal relationship via node feature concatenation and the non-linear mapping by MLP.}
On the other hand, 
we employ attention operations within the pairwise encoder to capture the contextual relationships (including persistent dependency and orientation asymmetry).

More specifically, the pairwise encoder module consists of the following parts (see Appendix Fig. \ref{fig:pem} for diagrammatic illustration):
\begin{enumerate}[leftmargin=*]
    \item \textit{Pairwise Feature Initialization.} The initial step is to concatenate the node features \textcolor{black}{from the previous node feature encoder module} for every pair of nodes.
    Subsequently, we employ a three-layer MLP to convert each concatenated vector $\mathcal{P}_{ijl} \in \mathbb{R}^{2h}$ to an $h$-dimensional raw pairwise feature, i.e., $\mathcal{P}_{ijl}^1 =  \mathrm{MLP}([\mathcal{F}_{il}; \mathcal{F}_{jl}])$. It is designed to capture the intricate relations that exist inside the pairs of nodes.
    \item \textit{Unidirectional Multi-Head Attention.} In order to model the external information, we employ an attention mechanism where the query is composed of the aforementioned $d^2$ $h$-dimensional raw pairwise features, while the keys and values consist of $h$-dimensional features of $d$ individual nodes, i.e., $\mathcal{P}^2 = \mathrm{MultiHeadAttention}(\mathcal{P}^1, \mathcal{F}, \mathcal{F})$.
Note that, this attention operation is unidirectional, which means we only calculate cross attention from raw pairwise features $\mathcal{P}^1$ to node features $F$.
This design is meant to capture \textcolor{black}{both pair-wise and node-wise information (as both are critical to model the causality, as discussed in Sec. \ref{sec:int}) while at the same time to maintain a reasonable computational cost.}
\item \textit{Final Processing.} Following the widely-adopted transformer architecture, we incorporate a residual structure and a dropout layer after the previous part, i.e., $\mathcal{P}^3 = \mathrm{Norm}(\mathcal{P}^1 + \mathcal{P}^2)$.
Finally, we introduce a three-layer MLP to further capture intricate patterns and non-linear relationships between the input embeddings, as well as to more effectively process the information from the attention mechanism: $\mathcal{P} = \mathrm{Norm}(\mathrm{MLP}(\mathcal{P}^3) + \mathcal{P}^3)$. 
It yields the final pairwise feature tensor $\mathcal{P} \in \mathbb{R}^{d \times d \times n \times h}$.
\end{enumerate}


\subsection{Learning Targets} \label{sec:met:lic}

As mentioned above, our learning target is a combination of the skeleton and the set of v-structures, which together represent an MEC. 
Two separate neural (sub-)networks are trained for these two targets.

\textbf{Skeleton Prediction.} 
As the persistent dependency between pairs of nodes determines the existence of edges in the skeleton, the pairwise features correspond to edges in the skeleton naturally.
Therefore, for the skeleton learning task, we initially employ a max-pooling layer over the observation dimension to obtain a single vector $\mathcal{S}_{ij} \in \mathbb{R}^{h}$ for each pair of nodes, i.e., $\mathcal{S}_{ij} = \max_{k} \mathcal{P}_{ijk}$.
Then, a linear layer and a sigmoid function are applied to map the pairwise features to the final prediction of edges, i.e., $S_{ij} = \mathrm{Sigmoid}(\mathrm{Linear}(\mathcal{S}_{ij}))$.
Our learning label, the undirected graph representing the skeleton, can be easily calculated by summing the adjacency of the DAG $G$ and its transpose $G^T$.
Therefore, our learning target for the skeleton prediction task can be formulated as $\min \mathcal{L}(S, G + G^T)$,
where $\mathcal{L}$ is the popularly used binary cross-entropy loss function.


\textbf{V-structure Prediction.}
A UT $\langle X_i, X_k, X_j \rangle$ is a v-structure when $\exists \mathbf{S}$, such that $X_k \notin \mathbf{S}$ and $X_i \perp X_j | \mathbf{S}$.
Motivated by this, we concatenate the corresponding pairwise features of the pair $\langle X_i, X_j \rangle$ with the node features of $X_k$ as the feature for each UT $\langle X_i, X_k, X_j \rangle$ after a max-pooling along the observation dimension, i.e., $\mathcal{U}_{kij} = [\max_l \mathcal{P}_{ijl};\max_l \mathcal{F}_{kl}]$.
After that, we use a three-layer MLP with a sigmoid function to predict the existence of v-structures among all UTs, i.e., $\mathcal{U}_{kij} = \mathrm{Sigmoid}(\mathrm{MLP}( \mathcal{U}_{kij}))$.
Given a data sample of $d$ nodes, it outputs a third-order tensor of shape $\mathbb{R}^{d \times d \times d}$, namely v-tensor, corresponding to the predictions of the existence of v-structures.
The v-tensor label can be obtained by $\mathcal{V}_{kij} = G_{ik} G_{jk} (1 - G_{ij})(1 - G_{ji})$,
where $\mathcal{V}_{kij}$ indicates the existence of v-structure $X_i \rightarrow X_k \leftarrow X_j$.
Therefore, the learning target for the v-structure prediction task can be formulated as $\min \mathcal{L}_{UT}(\mathcal{U}, \mathcal{V})$,
where $\mathcal{L}_{UT}$ is the binary cross-entropy loss masked by UTs, i.e., we only calculate such loss on the valid UTs.
In our current implementation, the parameters of the feature encoders are fine-tuned from the skeleton prediction task, as the UTs to be classified are obtained from the predicted skeleton and the skeleton prediction can be seen as a general pre-trained task.

Note that neural networks with our learning targets have a theoretical guarantee for correctness in asymptotic sense, as mentioned around the end of Sec. \ref{sec:met:lim}. 

\subsection{Post-Processing} 
\textcolor{black}{Although our method theoretically guarantees asymptotic correctness, conflicts in predicted v-structures might occasionally occur in practice. Therefore,} in the post-processing stage, we apply a straightforward heuristic to resolve the potential conflicts and cycles among predicted v-structures following previous work \citep{dai2023ml4c}.
\textcolor{black}{After that, we use an improved version of Meek rules \citep{meek1995causal,tsagris2019bayesian} to obtain other MEC-identifiable edges without introducing extra cycles.}
Combining the skeleton from the skeleton predictor model with all MEC-identifiable edge directions, we get the CPDAG predictions.

We provide a more detailed description of the post-processing process in Appendix Sec. \ref{app:post}. It is worth noting that our current design of post-processing is a very conservative one, and this module is also non-essential in our whole framework; see Appendix Sec. \ref{app:post} for more discussions and evidences.
\color{black}

\section{Experiments} \label{sec:exp}

\begin{table*}[!tb]
\centering
\begin{threeparttable}
\caption{\textbf{General comparison of SiCL and other methods}. The average performance results in three runs are reported for SiCL. GES takes more than 24 hours per graph on WS-L-G. SLdicso is unsuitable on non-linear-Gaussian data. \textcolor{black}{Full results on all metrics are provided in Appendix Tab. \ref{tab:epder}}.}
\label{tab:epders}
\begin{tabular}{cccccccccccc}
\toprule
 \multirow{2}{*}{Method} & \multicolumn{2}{c}{WS-L-G}& \multicolumn{2}{c}{SBM-L-G}& \multicolumn{2}{c}{WS-RFF-G}& \multicolumn{2}{c}{SBM-RFF-G}& \multicolumn{2}{c}{ER-CPT-MC} \\
 & s-F1$\uparrow$ & o-F1$\uparrow$ &s-F1$\uparrow$ & o-F1$\uparrow$&s-F1$\uparrow$ & o-F1$\uparrow$&s-F1$\uparrow$ & o-F1$\uparrow$&s-F1$\uparrow$ & o-F1$\uparrow$\\
\midrule
 PC & $30.4 $ & $16.0 $& $58.8$&$35.9$&$36.1$&$16.1$&$57.5$&$34.2$&$82.2$&$40.6$ \\
 GES & * & * & $70.8$& $55.0$&$41.7$&$23.6$&$56.5$&$38.0$&$82.1$&$42.4$\\
 NOTEARS & $33.3 $ & $31.5$&$80.1$&$77.8$&$37.7$&$33.4$&$55.6$&$48.5$&$16.7$&$0.6$ \\
 DAG-GNN & $35.5$ & $32.7$ &$66.2$&$62.5$&$33.2$&$28.9$&$47.1$&$40.6$&$24.8$&$3.7$\\
 GRAN-DAG & $16.6$&$11.7$&$22.6$&$14.4$&$4.7$&$1.1$&$17.4$&$3.8$&$40.8$&$7.3$ \\
 GOLEM &$30.0$ &$19.3$&$68.5$&$65.2$&$27.6$&$17.7$&$41.1$&$24.8$&$37.6$&$9.3$& \\
 SLdisco &$0.1$ &$0.1$&$1.9$&$1.2$&*&*&*&*&*&* \\
 AVICI & $39.9 $ & $35.8$ & $84.3$ & $81.6$& $47.7$& $45.2$& $76.6$& $72.7$& $76.9$& $57.6$\\
 SiCL & $\mathbf{44.7} $ & $\mathbf{38.5} $& $\mathbf{85.8} $ & $\mathbf{82.7} $ & $\mathbf{51.8}  $ & $ \mathbf{46.3}$ & $ {\mathbf{82.1}}$ & $\mathbf{78.0}$ &$\mathbf{84.2}$ & $\mathbf{59.9}$ \\
\bottomrule
\end{tabular}
\end{threeparttable}
\end{table*}
In this section, we report the performance of SiCL on both synthetic and real-world benchmarks, followed by an ablation study. 
More results and discussions about \textcolor{black}{time cost}, generality, and acyclicity are deferred to Appendix Sec. \ref{sec:app:exp:e}, due to page limit.

\subsection{Experiment Design}
\textbf{Metrics.} We profile a causal discovery method's performance using the following two tasks:

\emph{Skeleton Prediction}: 
Given a data sample $D$ of $d$ variables, for each variable pair, we want to infer if there exists direct causation between them. The standard metric \textbf{s-F1} (short for \textbf{skeleton-F1}) is used, which considers skeleton prediction as a binary classification task over the $d(d-1)/2$ variable pairs. 
For completion, we also report classification accuracy results. 
For methods with probabilistic outputs, AUC and AUPRC scores are also measured.



\emph{CPDAG Prediction}: 
\textcolor{black}{Given a data sample $D$ of $d$ variables, for each pair of variables, we aim to determine if there is direct causality between them. If so, we then assess whether the causal direction is MEC-identifiable, and if it is, we attempt to infer the specific causal direction.}
As this task involves both directed and undirected edge prediction, we use \textbf{SHD} (Structural Hamming Distance) to measure the difference between the true CPDAG and the inferred CPDAG. Besides that, we also measure the \textbf{o-F1} (short for \textbf{orientation-F1}) of the directed sub-graph of the inferred CPDAG (compared against the directed sub-graph of the true CPDAG), which focuses on capturing the inference method's orientation capability in identifying \textit{MEC-identifiable} causal edges. 
\textcolor{black}{Finally, we calculate the \textbf{v-F1} score, where the F1 score is based on the set of v-structures in the true CPDAG as the positive instances (from all ordered triples), and the v-structures in the inferred CPDAG as the positive predictions.}

\textbf{Testing Data.} 
A testing instance consists of a groundtruth causal graph $G$, the structural equations $f_i$ and noise variables $\epsilon_i$ for each variable $X_i$, and an i.i.d. sample $D$. We use two categories of testing instances in our experiments:

\textit{Analytical Instances}: 
where $G$ is sampled from a DAG distribution $\mathcal{G}$, and $\{f_i,\epsilon_i\}$ sampled from a structural-equation distribution $\mathcal{F}$ and a noise meta-distribution $\mathcal{N}$. 
We consider three random graph distributions for $\mathcal{G}$: Watts-Strogatz (WS), Stochastic Block Model (SBM), Erdos-Rényi (ER); and three $\mathcal{F}$'s: random linear (L), Random Fourier Features (RFF), and conditional probability table (CPT). $\mathcal{N}$ is a uniform distribution over Gaussian's for continuous data, and a Dirichlet distribution over Multinomial Categorical distributions for discrete data. 
\textcolor{black}{We examine five combinations of testing instances: \textbf{WS-L-G}, \textbf{SBM-L-G}, \textbf{WS-RFF-G}, \textbf{SBM-RFF-G}, and \textbf{ER-CPT-MC}.}

\textit{Real-world Instance}:
The classic dataset \textbf{Sachs} is used \textcolor{black}{to evaluate performance in real-world scenarios}. 
It consists of a data sample recording the concentration levels of 11 phosphorylated proteins in 853 human immune system cells, and of a causal graph over these 11 variables identified by \cite{sachs2005causal} based on expert consensus and biology literature.

\textbf{Algorithms.}
As baselines, we compare with a series of representative unsupervised methods, including \textbf{PC} (using the recent Parallel-PC variation by ~\citet{le2016fast}), 
\textbf{GES}~\citep{chickering2002optimal}, 
\textbf{NOTEARS}~\citep{zheng2018dags}, 
\textbf{GOLEM}~\citep{ng2020role}, 
\textbf{DAG-GNN}~\citep{yu2019dag},
\textbf{GRANDAG}~\citep{Lachapelle2020Gradient-Based}, \textbf{SLdisco} \citep{petersen2023causal} as well as \textbf{AVICI}~\citep{lorchamortized}, a DNN-based SCL method regarded as current state-of-the-art method.

For our method, besides the full \textbf{SiCL} implementation as described by Sec. \ref{sec:algo}, 
we also implement 
\textbf{SiCL-Node-Edge}, which predicts the causal graph using the node features and can be regarded as equivalent to AVICI, and 
\textbf{SiCL-no-PF}, which skips pairwise feature extraction and predicts the skeleton and v-tensor using node-wise features (see Appendix Fig. \ref{fig:abl}). 
Notably, SiCL contains 2.8M parameters, while SiCL-Node-Edge and SiCL-no-PF contain 3.2M parameters, because SiCL contains fewer layers on the node feature encoder to eliminate potential bias from size difference. 

For DNN-based SCL methods, the DNNs are trained with synthetic data where the causal graphs follow the Erdos-Rényi (ER) and Scale-Free (SF) models and the structural equations and noise variables follow the same distribution type as the corresponding testing data. 
Therefore, the disparities between the causal graph distribution at training and testing time help to examine the generality of SiCL in \textbf{OOD} settings to some extent.
More details of the experimental setting are presented in Appendix Sec. \ref{sec:app:exp:set}.

\subsection{Results on Synthetic Dataset} \label{sec:exp:gp}

We conduct a comprehensive comparison of SiCL with various baselines in both skeleton prediction and CPDAG prediction tasks.
The main results of metrics skeleton-F1 and orientation-F1 are presented in Tab. \ref{tab:epders}, and results on full metrics are provided in Appendix Tab. \ref{tab:epder}.
On continuous data, DNN-based SCL methods (i.e., AVICI and SiCL) demonstrate consistent and obvious advantages over traditional approaches.
SiCL consistently outperforms the other methods on both skeleton prediction task and CPDAG prediction task.
On the other hand, some unsupervised methods achieve comparable performance among DNN-based SCL methods on the discrete data ER-CPT-MC. 
Nonetheless, our proposed SiCL emerges as the top performer, further substantiating its superiority in addressing the causal learning problem.

\subsection{Results on Real-world Dataset} \label{sec:exp:erd}
\begin{table}
\centering
\caption{Comparison on Sachs dataset.}
\label{tab:sachs}
    \resizebox{\linewidth}{!}{%
    \begin{threeparttable}
\begin{tabular}{@{}ccccc@{}}
\toprule
 \multirow{2}{*}{Method} & \multicolumn{2}{c}{Skeleton Prediction }& \multicolumn{2}{c}{CPDAG Prediction} \\
 &  s-F1$\uparrow$ & s-Acc.$\uparrow$ & SHD$\downarrow$ & \#v-struc.$\downarrow$\\ \midrule
 PC& $68.6$& $80.0$ &$19$&$12$\\
 GES & $70.6$ & $81.8$ &$19$&$8$\\
 DAG-GNN & $21.1$ & $72.7$&$15$&$\mathbf{0}$ \\
 NOTEARS & $11.1$ & $70.9$ &$16 $ & $ \mathbf{0}$ \\
  GRAN-DAG & $45.5 $ & $78.2 $ &$ 12 $ & $\mathbf{0}$ \\
  GOLEM & $ 36.4$ & $ 74.5$ &$ 14 $ & $\mathbf{0}$ \\
AVICI & $66.7 $&$ 83.5$  &$18 $&$ 14$\\ 
SiCL & $\mathbf{71.4}$ & $\mathbf{86.8}$&$\mathbf{6}$&$\mathbf{0}$   \\ 
\bottomrule
\end{tabular}
\end{threeparttable}
}
\end{table}
To assess the practical applicability of SiCL, we conduct a comparison using the real-world dataset Sachs.
The discretized Sachs data obtained from the bnlearn library\footnote{https://www.bnlearn.com/} is used.
The DNN-based SCL methods are trained on random synthetic graphs, making this also an \textbf{OOD} prediction task.
The results are provided in Tab. \ref{tab:sachs}.

For the skeleton prediction task, SiCL performs the best, albeit with a modest gap (generally 1$\sim$3 scores higher than the runners-up). For the CPDAG prediction task, SiCL performs significantly better than all other methods (reducing SHD from 12 to 6, against the second best). Interestingly, the true causal DAG of the Sachs benchmark actually contains no v-structure, so any predicted v-structure is an error. We see that methods competitive with SiCL in skeleton prediction (AVICI, PC, GES) mistakenly predicted a large number of v-structures on the Sachs data, while SiCL correctly predict zero v-structure.

\begin{table}[!tb]
\centering
\begin{threeparttable}
\caption{Ablation study of SiCL components. Full metrics are available in Appendix Tab. \ref{tab:fcplg}.}
\label{tab:cplg}
\begin{tabular}{ccccc}
\toprule
 \multirow{2}{*}{Method} & \multicolumn{2}{c}{WS-L-G}& \multicolumn{2}{c}{SBM-L-G} \\
 & s-F1$\uparrow$ & o-F1$\uparrow$ &s-F1$\uparrow$ & o-F1$\uparrow$\\\midrule
 SiCL-Node-Edge & $39.9$ & $35.8$ & $84.3$ & $81.6$ \\
 SiCL-no-PF &$42.4$ & $37.9$ & $85.5$ & $82.2$ \\ 
 SiCL & $\mathbf{44.7}$ & $\mathbf{38.5}$ &$\mathbf{85.8}$& $\mathbf{82.7}$ \\
\bottomrule
\end{tabular}
\end{threeparttable}
\end{table}

\subsection{\textcolor{black}{Ablation Study}}

\begin{figure}[t]
    \centering
    \includegraphics[width=\linewidth]{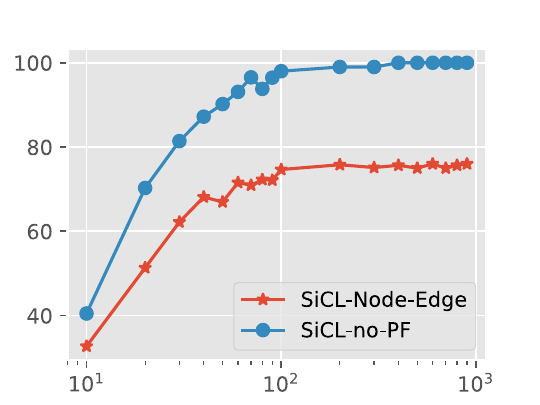}
    \caption{Comparison of SiCL-Node-Edge and SiCL-no-PF in o-F1 trend as observation samples increase on a constructed dataset.}
    \label{fig:cto}
\end{figure}
\textbf{Effectiveness of Learning Identifiable Structures.} \label{sec:exp:elis}As discussed in Section \ref{sec:met:lic}, SiCL focuses on learning MEC-identifiable causal structures rather than directly learning the adjacency matrix.
To verify the effectiveness of this idea,
we compare SiCL-Node-Edge with SiCL-no-PF.
These two models share a similar node feature encoder architecture but have different learning targets: SiCL-Node-Edge predicts the adjacency matrix, while the SiCL-no-PF predicts the skeleton and v-tensor.
The results are shown in Table \ref{tab:cplg}.
Consistently, SiCL-no-PF demonstrates higher performance on both skeleton and CPDAG prediction tasks. 
This observation echoes our theoretical conclusion regarding the necessity and benefits of learning identifiable causal structures to improve overall performance.

To further underscore the significance of learning identifiable causal structures (especially in the asymptomatic sense), we conduct a comparative analysis using a specially constructed dataset, which contains six nodes forming an independent v-structure and a UT. 
Figure \ref{fig:cto} illustrates that the orientation F1 scores of CPDAG predictions from SiCL-Node-Edge suffer from an unavoidable error and do not improve with the addition of more observational samples. 
In contrast, predictions from SiCL-no-PF reach perfect accuracy, confirming the value of learning identifiable causal structures.

\textbf{Effectiveness of Pairwise Representation.} To assess the effectiveness of pairwise representation, we compare the full version of SiCL with a variant lacking pairwise features (SiCL-no-PF). As shown in Table \ref{tab:cplg}, the full-version SiCL consistently outperforms SiCL-no-PF in both skeleton prediction and CPDAG prediction tasks. Notably, we have intentionally set the model size of the full-version SiCL (2.8M parameters) to be smaller than that of SiCL-no-PF (3.2M parameters) so as to avoid any potential advantage from increased model complexity brought by the pairwise feature encoder module. The observed performance gains in this case underscore the critical role of pairwise features in identifying causal structures. Additionally, we conduct further comparisons across more diverse settings, with results detailed in Appendix Sec. \ref{sec:mpc}. These results demonstrate even more pronounced improvements in favor of SiCL, reinforcing the importance of pairwise representations in causal discovery.

\section{Conclusion}

We proposed SiCL, a novel DNN-based SCL approach designed to predict the corresponding skeleton and a set of v-structures. We showed that such design do not suffer from the (non-)identifiability limit \textcolor{black}{that exists in current architectures}. Moreover, SiCL is equipped with a pairwise encoder module to explicitly model relationships between node-pairs. 
Experimental results validated the effectiveness of these ideas. 

This paper also introduces a few interesting open problems.
The proposed DNN model works in the canonical setting under the classic MEC theory, in which the skeleton and v-structures are the identifiable structure.
It can be an interesting future-work direction to explore how to learn other identifiable causal structure in other assumption settings following the same principle.
Due to the inherent complexity of DNNs, the explanation of the decision mechanism of our model remains an open question. 
Therefore, future work could consider to explore how decisions are made within the networks and provide some insights for traditional methods. 
Moreover, the proposed pairwise encoder modules needs $O(d^3)$ computational complexity, which may restrict its current application to scenarios with huge number of nodes.
Future work could focus on simplifying these operations or exploring features with less complexity (e,g., low rank features) to reduce the overall computational cost.

\bibliography{main}

\appendix
\renewcommand\thesection{A\arabic{section}}
\renewcommand{\thetable}{A\arabic{table}}
\renewcommand{\thefigure}{A\arabic{figure}}
\renewcommand{\thealgorithm}{A\arabic{algorithm}}

\section*{Checklist}



 \begin{enumerate}

 \item For all models and algorithms presented, check if you include:
 \begin{enumerate}
   \item A clear description of the mathematical setting, assumptions, algorithm, and/or model. [Yes] 
   \item An analysis of the properties and complexity (time, space, sample size) of any algorithm. [Yes] 
   \item (Optional) Anonymized source code, with specification of all dependencies, including external libraries. [Yes] 
 \end{enumerate}

 \item For any theoretical claim, check if you include:
 \begin{enumerate}
   \item Statements of the full set of assumptions of all theoretical results. [Yes] 
   \item Complete proofs of all theoretical results. [Yes] 
   \item Clear explanations of any assumptions. [Yes] 
 \end{enumerate}

 \item For all figures and tables that present empirical results, check if you include:
 \begin{enumerate}
   \item The code, data, and instructions needed to reproduce the main experimental results (either in the supplemental material or as a URL). [Yes] 
   \item All the training details (e.g., data splits, hyperparameters, how they were chosen). [Yes] 
         \item A clear definition of the specific measure or statistics and error bars (e.g., with respect to the random seed after running experiments multiple times). [Yes] 
         \item A description of the computing infrastructure used. (e.g., type of GPUs, internal cluster, or cloud provider). [Yes] 
 \end{enumerate}

 \item If you are using existing assets (e.g., code, data, models) or curating/releasing new assets, check if you include:
 \begin{enumerate}
   \item Citations of the creator If your work uses existing assets. [Yes] 
   \item The license information of the assets, if applicable. [Not Applicable]
   \item New assets either in the supplemental material or as a URL, if applicable. [Not Applicable]
   \item Information about consent from data providers/curators. [Yes]
   \item Discussion of sensible content if applicable, e.g., personally identifiable information or offensive content. [Not Applicable]
 \end{enumerate}

 \item If you used crowdsourcing or conducted research with human subjects, check if you include:
 \begin{enumerate}
   \item The full text of instructions given to participants and screenshots. [Not Applicable]
   \item Descriptions of potential participant risks, with links to Institutional Review Board (IRB) approvals if applicable. [Not Applicable]
   \item The estimated hourly wage paid to participants and the total amount spent on participant compensation. [Not Applicable]
 \end{enumerate}

 \end{enumerate}

\newpage 
\clearpage
\onecolumn

\begin{algorithm}[!tb]
\caption{SiCL Workflow for Predicting Causal Structures}
\label{alg:workflow}
\begin{algorithmic}
\STATE \textbf{Procedure} INFERENCE(data, target)
\STATE Calculate node features with node encoder
\STATE Calculate pairwise features with pairwise encoder following Sec. \ref{sec:fem}
\IF {target is skeleton }
\STATE Calculate skeleton with Sec. \ref{sec:met:lic}
\ELSE
\STATE Calculate v-structures with Sec. \ref{sec:met:lic}
\ENDIF
\STATE \textbf{End Procedure}
\STATE
\STATE \textbf{Procedure} TRAINING\_PHASE()
\STATE  skeleton\_predictor $\leftarrow$ init\_skeleton\_predictor()

\STATE  Sample graphs and corresponding data
\STATE Training the skeleton predictor with INFERENCE(data, skeleton)
\STATE Training the v-structure predictor with INFERENCE(data, v-structure), with feature encoders fine-tuned from skeleton predictor
\STATE \textbf{End Procedure}
\STATE
\STATE \textbf{Procedure} TESTING\_PHASE(test\_data)
\STATE  Calculate predicted skeleton with the trained skeleton predictor
\STATE  Calculate predicted v-structures with the trained v-structure predictor
\STATE Combine predicted skeleton and v-structures to obtain predicted CPDAG
\STATE \textbf{End Procedure}
\end{algorithmic}
\end{algorithm}

\begin{figure*}[t]
\centering
    \includegraphics[width=\linewidth]{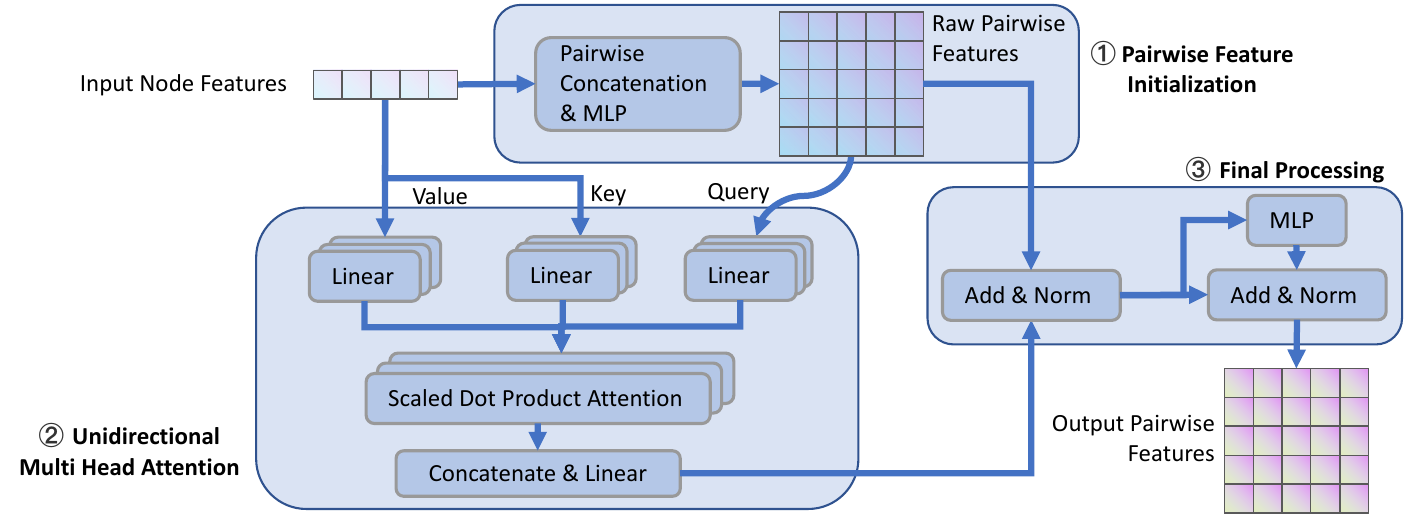}
    \caption{Illustration of the pairwise encoder module. \textcolor{black}{In Part \ding{172}, it initializes raw pairwise features. In Part \ding{173}, a unidirectional attention is applied to utilized information from node features and pairwise features. In Part \ding{174}, an MLP and residual connection is used to yield final pairwise features.}}
    \label{fig:pem}
    
\end{figure*}

\section{Theoretical Guarantee} \label{sec:app:tg}
In this section, we delve into the theoretical analysis concerning the asymptotic correctness of our proposed model with respect to the sample size. Sec. \ref{sec:da} lays out the essential definitions and assumptions pertinent to the problem under study. Following this, from Sec. \ref{sec:sl} to \ref{sec:ol}, we rigorously demonstrate the asymptotic correctness of the neural network model. Finally, in Sec. \ref{sec:d}, we engage in a detailed discussion about the practical advantages and superiority of neural network models.

\subsection{Definitions and Assumptions} \label{sec:da}
As outlined in Sec. \ref{sec:bg}, a Causal Graphical Model is defined by a joint probability distribution $P$ over $d$ random variables $X_1, X_2, \cdots, X_{d}$, and a DAG $G$ with $d$ vertices representing the $d$ variables.
An observational dataset $D$ consists of $n$ records and $d$ columns, which represents $n$ instances drawn i.i.d. from $P$. 
In this work, we assume causal sufficiency:
\begin{Assumption}[Causal Sufficiency] \label{ass:cs}
    There are no latent common causes of any of the variables in the graph. 
\end{Assumption}
Moreover, we assume the data distribution $P$ is Markovian to the DAG $G$:
\begin{Assumption}[Markov Factorization Property]\label{ass:mk}
      Given a joint probability distribution $P$ and a DAG $G, P$ is said to satisfy Markov factorization property w.r.t. $G$ if $P:=$ $P\left(X_1, X_2, \cdots, X_d\right)=\prod_{i=1}^d P\left(X_i \mid \mathrm{pa}_i^G\right)$, where $\mathrm{pa}_i^G$ is the parent set of $X_i$ in $G$.
\end{Assumption}
It is noteworthy that the Markov factorization property is equivalent to the Global Markov Property (GMP) \citep{lauritzen1996graphical}, which is
\begin{Definition}[Global Markov Property (GMP)]
    $P$ is said to satisfy GMP (or Markovian) w.r.t. a DAG $G$ if $X \perp_G Y|Z \Rightarrow X \perp Y| Z$. Here $\perp_G$ denotes d-separation, and $\perp$ denotes statistical independence. 
\end{Definition}
GMP indicates that any d-separation in graph $G$ implies conditional independence in distribution $P$. We further assume that $P$ is faithful to $G$ by:
\begin{Assumption}[Faithfulness]\label{ass:f}
Distribution $P$ is faithful w.r.t. a DAG $G$ if $X \perp Y\left|Z \Rightarrow X \perp_G Y\right| Z$.
\end{Assumption}

\begin{Definition}[Canonical Assumption] \label{ass:ca}
    We say our settings satisfy the canonical assumption if the Assumptions \ref{ass:cs} - \ref{ass:f} are all satisfied.
\end{Definition}
We restate the definitions of skeletons, Unshielded Triples (UTs) and v-strucutres as follows.
\begin{Definition}[Skeleton]
    A skeleton $E$ defined over the data distribution $P$ is an undirected graph where an edge exists between $X_i$ and $X_j$ if and only if $X_i$ and $X_j$ are always dependent in $P$, i.e., $\forall Z \subseteq\left\{X_1, X_2, \cdots, X_d\right\} \backslash \left\{X_i, X_j \right\}$, we have $X_i \nperp X_j | Z$.
\end{Definition}
Under our assumptions, the skeleton is the same as the corresponding undirected graph of $G$ \citep{spirtes2000causation}. 
\begin{Definition}[Unshielded Triples (UTs) and V-structures]
A triple of variables $X, T, Y$ is an Unshielded Triple (UT) denoted as $\langle X, T, Y \rangle$, if $X$ and $Y$ are both adjacent to $T$ but not adjacent to each other in the DAG $G$ or the corresponding skeleton.
It becomes a v-structure denoted as $X \rightarrow T \leftarrow Y$, if the directions of the edges are from $X$ and $Y$ to $T$ in $G$.
\end{Definition}
We introduce the definition of separation set as:
\begin{Definition}[Separation Set]
    For a node pair $X_i$ and $X_j$, a node set $Z$ is a separation set if $X_i \perp X_j | Z $. Under faithfulness assumption, a separation set $Z$ is a subset of variables within the vicinity that d-separates $X_i$ and $X_j$.
\end{Definition}

Finally, we assume a neural network can be used as a universal approximator in our settings.
\begin{Assumption}[Universal Approximation Capability]
    A neural network model can be trained to approximate a function under our settings with arbitary accuracy. \label{ass:uac}
\end{Assumption}

\subsection{Skeleton Learning} \label{sec:sl}
In this section, we prove the asymptotic correctness of neural networks on the skeleton prediction task by constructing a perfect model and then approximating it with neural networks. 
For the sake of convenience and brevity in description, we define the skeleton predictor as follows. 
\begin{Definition}[Skeleton Predictor]
    Given observational data $D$, a skeleton predictor is a predicate function with domain as observational data $D$ and predicts the adjacency between each pair of the vertices.
\end{Definition}
Now we restate the Remark from \cite{ma2022ml4s} as the following proposition. 
It proves the existence of a perfect skeleton predictor by viewing the skeleton prediction step of PC \citep{spirtes2000causation} as a skeleton predictor, which is proved to be sound and complete.
\begin{Proposition}[Existence of a Perfect Skeleton Predictor]
There exists a skeleton predictor that always yields the correct skeleton with sufficient samples in $D$. \label{prop:epsp}
\end{Proposition}
\begin{proof}
    We construct a skeleton predictor $SP$ consisting of two parts by viewing PC \citep{spirtes2000causation} as a skeleton predictor. 
    In the first part, it extracts a pairwise feature $\boldsymbol{x}_{i j}$ for each pair of nodes $X_i$ and $X_j$:
    \begin{align}
        \boldsymbol{x}_{i j}=\min _{Z \subseteq V \backslash\left\{X_i, X_j\right\}}\left\{X_i \sim X_j \mid Z\right\}, \label{equ:sp1}
    \end{align}
    where $\left\{X_i \sim X_j \mid Z\right\} \in [0, 1] $ is a scalar value that measures the conditional dependency between $X_i$ and $X_j$ given a node subset $Z$. 
    Consequently, $\boldsymbol{x}_{i j} > 0$ indicates the persistent dependency between the two nodes.
    
    In the second part, it predicts the adjacency based on $\boldsymbol{x}_{i j}$:
    \begin{align}
        \left(X_i,  X_j\right)= \begin{cases} 1 \text { (adjacent) } & \boldsymbol{x}_{i j} \neq 0 \\ 0 \text { (non-adjacent) } & \boldsymbol{x}_{i j} = 0\end{cases} \label{equ:sp2}
    \end{align}

    Now we prove that $SP$ always yields the correct skeleton by proving the absence of false positive predictions and false negative predictions. Here, false positive prediction denotes $SP$ predicts a non-adjacent node pair as adjacent and false negative predictions denote $SP$ predicts an adjacent node pair as non-adjacent.

    \begin{itemize}[leftmargin=*]
        \item \textbf{False Positive.} Suppose $X_i, X_j$ are non-adjacent. Under the Markovian assumption, there exists a set of nodes $Z$ such that $\left\{X_i \sim X_j \mid Z\right\} = 0$ and hence $\boldsymbol{x}_{ij} = 0$. According to Eq. (\ref{equ:sp2}), $SP$ will always predicts them as non-adjacent.
        \item \textbf{False Negative}. Suppose $X_i, X_j$ are adjacent. Under the faithfulness assumption, for any $Z \in V \backslash \left\{X_i, X_j\right\}, \left\{X_i \sim X_j \mid Z\right\} > 0$, which implies $\boldsymbol{x}_{ij} > 0$. Therefore, $SP$ always predicts them as adjacent.
    \end{itemize}

    Therefore, $SP$ never yields any false positive predictions or false negative predictions under the Markovian assumption and faithfulness assumption, i.e., it always yields the correct skeleton.
\end{proof}


With the existence of a perfect skeleton predictor, we prove the correctness of neural network models with sufficient samples under our assumptions.
\begin{Theorem}
Under the canonical assumption and the assumption that neural network can be used as a universal approximator (Assumption \ref{ass:uac}),
there exists a neural network model that always predicts the correct skeleton with sufficient samples in $D$.
\end{Theorem}
\begin{proof}
    From Proposition \ref{prop:epsp}, there exists a perfect skeleton predictor that predicts the correct skeleton. 
    Thus, according to the Assumption \ref{ass:uac}, a neural network model can be trained to approximate the perfect skeleton prediction hence predicts the correct skeleton. 
\end{proof}

\subsection{Orientation Learning} \label{sec:ol}
Similarly to the overall thought process in Sec. \ref{sec:sl}, in this section we prove the asymptotic correctness of neural networks on the v-structure prediction task by constructing a perfect model and then approximating it with neural networks. 
\begin{Definition}[V-structure Predictor]
    Given observational data $D$ with sufficient samples from a $BN$ with vertices $V = \{X_1, \dots, X_p\}$, a v-structure predictor is a predicate function with domain as observational data $D$ and predicts existence of the v-structure for each unshielded triple.
\end{Definition}
The following proposition proves the existence of a perfect v-structure predictor by viewing the orientation step of PC \citep{spirtes2000causation} as a v-structure predictor.
\begin{Proposition}[Existence of a Perfect V-structure Predictor]
Under the Markov assumption and faithfulness assumption, there exists skeleton predictor that always yields the correct skeleton. \label{prop:epvp}
\end{Proposition}
\begin{proof}
    We construct a v-structure predictor $VP$ consisting of two parts by viewing PC \citep{spirtes2000causation} as a v-structure predictor. 
        In the first part, it extracts a boolean feature $\boldsymbol{z}_{kij}$ for each UT $\langle X_i, X_k, X_j \rangle$:
    \begin{align}
        \boldsymbol{z}_{kij} = (X_k \in Z), \text{ where } Z \text{ is called as a sepset, i.e. } X_i\perp Y_j | Z.  \label{equ:vp}
    \end{align}
    \color{black}

Note that the sepset $Z$ always exists because the separation set of a UT always exists (See Lemma 4.1 in \cite{dai2023ml4c}).

In the second part, it predicts the v-structures based on $\boldsymbol{z}_{ijk}$:
\begin{align}
\langle X_i, X_k, X_j\rangle = \begin{cases} 0 \text { (not v-structure) } & \boldsymbol{z}_{k i j } = True \\ 1 \text { (v-structure) } & \boldsymbol{z}_{k i j } = False\end{cases} \label{equ:vp2}
\end{align}
Now we prove that $VP$ always yields the correct predictions of v-structures.
According to Theorem 5.1 on p.410 of \cite{spirtes2000causation}, assuming faithfulness and sufficient samples, if a UT $\langle X_i, X_k, X_j \rangle$ is a v-structure, then $X_k$ does not belong to any separation sets of $(X_i, X_j)$; if a UT $\langle X_i, X_k, X_j \rangle$ is not a v-structure, then $X_k$ belongs to every separation sets of $(X_i, X_j)$. Therefore, we have $\boldsymbol{z}_{kij} = False$ if and only if $X_k$ is not in any separation set of $X_i$ and $X_j$, i.e., $\langle X_i, X_k, X_j \rangle$ is a v-structure.
\color{black}
\end{proof}

With the existence of a perfect v-structure predictor, we prove the correctness of neural network models with sufficient samples under our assumptions.
\begin{Theorem}
Under the canonical assumption and the assumption that neural network can be used as a universal approximator (Assumption \ref{ass:uac}), there exists a neural network model that always predicts the correct v-structures with sufficient samples in $D$.
\end{Theorem}
\begin{proof}
    From Proposition \ref{prop:epsp}, there exists a perfect skeleton predictor that predicts the correct v-structures. 
    Thus, according to the Assumption \ref{ass:uac}, a neural network model can be trained to approximate the perfect v-structure predictions hence predicts the correct v-structures. 
\end{proof}

\subsection{Discussion} \label{sec:d}
In the sections above, we prove the asymptotic correctness of neural network models by constructing theoretically perfect predictors.
These predictors both consist of two parts: feature extractors providing features $\boldsymbol{x}_{ij}$ and $\boldsymbol{z}_{ijk}$, and final predictors of adjacency and v-structures.
Even though they have a theoretical guarantee of the correctness with sufficient samples, it is noteworthy that they are hard to be applied practically.
For example, to obtain $\boldsymbol{x}_{ij}$ in Eq. (\ref{equ:sp1}), we need to calculate the conditional dependency between $X_i$ and $X_j$ given every node subset $Z \subseteq V \backslash\left\{X_i, X_j\right\}$.
Leaving aside the fact that the number of $Z$s itself presents factorial complexity, the main issue is that when $Z$ is relatively large, due to the curse of dimensionality, it becomes challenging to find sufficient samples to calculate the conditional dependency. 
This difficulty significantly hampers the ability to apply the constructed prefect predictors in practical scenarios.

Some existing methods can be interpreted as constructing more practical predictors.
Majority-PC (MPC) \citep{colombo2014order} achieves better performance on finite samples by modifying Eq. (\ref{equ:vp}) - (\ref{equ:vp2}) as:
    \begin{align}
        \boldsymbol{z}_{kij} = \frac{\left|\left\{ (X_k, Z) | \{ X_i \sim X_j | Z\} = 0 \wedge X_k \in Z \right\}\right|}{\left|\left\{ Z | \{ X_i \sim X_j | Z\} = 0 \right\}\right|},
    \end{align}
    and
\begin{align}
    \left\langle X_i, X_k, X_j\right\rangle= \begin{cases}0 \text { (not v-structure) } & \boldsymbol{z}_{i j k} > 0.5 \\ 1(\mathrm{v} \text {-structure) } & \boldsymbol{z}_{i j k} \leq 0.5,\end{cases}
\end{align}
    where $\left\{X_i \sim X_j \mid Z\right\} \in [0, 1] $ is a scalar value that measures the conditional dependency between $X_i$ and $X_j$ given a node subset $Z$, and $|\cdot|$ represents the cardinality of a set. 
\color{black}
Due to its more complex classification mechanism, it achieves better performance empirically. 
However, from the machine learning perspective, features from both the PC and MPC predictors are relatively simple.
As supervised causal learning methods, ML4S \citep{ma2022ml4s} and ML4C \citep{dai2023ml4c} provide more systematic featurizations by manual feature engineering and utilization of powerful machine learning models for classification.
While these methods show enhanced practical efficacy, their manual feature engineering processes are complex. 
In our paper, we utilize neural networks as universal approximators for learning the prediction of identifiable causal structures.
It not only simplifies the procedure but also potentially uncovers more nuanced and complex patterns within the data that manual methods might overlook.
It is noteworthy that the benefits of supervised causal learning using neural networks are also discussed elsewhere, as mentioned in SLdisco \citep{petersen2023causal} and CSIvA \citep{ke2023learning}.

\begin{algorithm}[!tb]
\caption{Post-processing}
\label{alg:post}
\begin{algorithmic}
\STATE {\bfseries Input:} weighted skeleton matrix $S$, weighted V-tensor $U$, threshold for skeleton $\tau_s$, threshold for v-structure $\tau_v$
\STATE {\bfseries Output:} predicted oriented edge set $\texttt{oriEdges}$, predicted skeleton $\texttt{skeleton}$
\STATE \textbf{Step 1:}  \\
// Obtaining a predicted skeleton by thresholding. \\
$\texttt{skeleton} = \{(i, j) | max(S_{ij}, S_{ji}) > \tau_s\}$ \\
// Obtaining raw v-structures $\texttt{vstructs}_\texttt{raw}$ by thresholding. \\
$\texttt{vstructs}_\texttt{raw} = \{(i, j, k) | (i, j) \in \texttt{skeleton} \text{ and } (i, k) \in \texttt{skeleton} \text{ and } (j, k) \notin \texttt{skeleton} \text{ and } max(U_{ijk}, U_{ikj}) > \tau_v\}$ \\
\STATE \textbf{Step 2:}  \\
// V-structure conflict resolving: discard any v-structure if there exists another conflicted v-structure with a higher predicted score, following \citep{dai2023ml4c}. \\
$\texttt{vstructs} = \{(i, j, k) \in \texttt{vstructs}_\texttt{raw}|\forall (i', j', k') \in \texttt{vstructs}_\texttt{raw}, (i' \neq k \text{ and }i' \neq j)\text{ or } (k'\neq i \text{ and } j' \neq i) \text{ or } U_{i'j'k'} < U_{ijk}\}$ 
\STATE \textbf{Step 3:}  \\
// Obtaining the predicted directed edge from $\texttt{vstructs}$. \\
$\texttt{oriEdges}_{\texttt{raw}} = \{(j, i)| \exists k, (i, j, k) \in \texttt{vstructs}\}$ \\
//Set a score for each edge with the highest v-structure's score containing this edge. \\
Set $\{p_{ij}\}$ such that $p_{ij} = max_v U_v$ for $v \in \texttt{oriEdges}_{\texttt{raw}}$ and $v \ni (i, j)$. \\
// If there exist any cycles, remove the edge with the smallest score in each cycle. \\
$\texttt{oriEdges} = \{(i, j) \in \texttt{oriEdges}_{\texttt{raw}} |\forall \text{cycle } C, (i, j) \notin C \text{ or } (\exists (i', j') \in C, p_{ij} > p_{i'j'})\}$ \\
\STATE \textbf{Step 4:} \\
// Meek rules: Add edges to $\texttt{oriEdges}$ for directed edges that (1) introducing the edges does not lead to cycles or new v-structures; (2) adding the opposite edges necessarily leads to cycles or new v-structures. \\
$\texttt{oriEdges} = \texttt{oriEdges} \cup \{(i, j) \in \texttt{skeleton}| (i, j) \text{ complies with Meek rules} \}$ \\
\end{algorithmic}
\end{algorithm}

\section{More Discussions on Identifiability and Causal Assumptions} \label{sec:dica}
\textbf{Advocation of Learning Identifiable Structures under All Settings.}
In this paper, we have to work on a concrete setting for demonstration purpose with concrete identifiable causal structures in this paper.
Nonetheless, we want to emphasize that the very concept of identifiability, as well as its ramifications in SCL, is indeed a general issue that is less bound to the issue of ``which causal structures are identifiable under which assumptions". 
The simple fact that in some situations the causal edge cannot be identified -- no matter what feature can be identified in that case -- this identifiability limit has a general effect on SCL. 
Unless the causal graph/edge itself becomes fully invariant/identifiable (a special case that is important but certainly not universally true), the presence of the identifiability limit entails a fundamental bias for a popular SCL model architecture (i.e., Node-Edge) that cannot be mitigated by larger model or bigger data at all. 
This ``identifiability-limit-causes-learning-error" effect is the main thesis of this paper, and we advocate to design neural networks that focus on learning the identifiable features (no matter what those features are). 
In other words, there is nothing stopping one from studying another setting where another feature is identifiable though, and in that case we would also advocate to learn that feature instead of v-structures. 
For example, if we assume canonical MEC assumptions and non-existence of causal-fork and v-structures, the identifiable causal structure becomes a kind of chains.
In that case, one may want to design neural networks that predict about causal chains. 

\textbf{Rationality of Canonical Assumptions.}
In this paper, we choose the canonical setting under the classic MEC theory, in which the skeleton and v-structures are the identifiable structure.
This setting includes the assumptions of the Markov and faithfulness conditions.
Unlike scenario-specific assumptions, such as those tied to a particular data-generating process, these assumptions are classic assumptions about causality that are often adopted as ``postulates" about some general aspects of the world. For example,
\begin{itemize}
    \item \cite{pearl2009causality} argues that stability (faithfulness) stems from the natural improbability of strict equality constraints among parameters, which aligns with the autonomy of causal mechanisms.
\item \cite{spirtes2001causation} support the Causal Faithfulness Condition (CFC) by noting that the exact cancellation of causal paths is highly improbable under natural conditions.
\item \cite{weinberger2018faithfulness} reinforces this argument, proposing that coincidences leading to CFC violations are rare and lack explanatory power, further justifying its adoption within a general modeling framework.
\end{itemize}
These considerations underscore the rationality and generality of the assumptions, making them a natural choice for our analysis.

\section{Details and Discussion about Post-processing} \label{app:post}
For comprehensive clarity, we provide a clear process about the post-processing algorithm in Alg. \ref{alg:post}.

\textbf{Discussion. }
It is worth noting that our design in post-processing is as conservative as possible. In fact, we simply adhere to the conventions in deep learning (i.e., thresholding) to obtain the skeleton and the initial v-structure set. Subsequently, we follow the conventions in constraint-based causal discovery methods to derive the final oriented edges. Therefore, we have not dedicated extensive efforts towards the meticulous design, nor do we intend to emphasize this aspect of our workflow.

The conflicts and cycles are not unique to SiCL; they are, in fact, common issues encountered by all constraint-based algorithms like PC. Moreover, it's worth noting that they never appear if the networks perform perfect. Therefore, the conflict resolving of v-structures and the removal of cycles are designed as fallback mechanisms to ensure the soundness of our workflow, rather than being central elements of our approach. To illustrate it, we experimented with an opposite variant (Intuitively, this is a bad choice) that prioritizes discarding the v-structure with the higher predicted probability. The minimal differences in outcomes between this variant and its counterpart, as detailed in Tab. \ref{tab:crm}, support our viewpoint that the conflict resolution process is of limited significance within our workflow. On the other hand, experimental results presented in Tab. \ref{tab:ccfc} underscore the infrequency of cycles in the predictions, reinforcing the non-essential nature of the cycle removal component.

\begin{table}[tb]
\centering
\begin{threeparttable}
\caption{The o-F1 comparison between the used conflict resolving method with an opposite variant.}
\label{tab:crm}
\begin{tabular}{@{}ccc@{}}
\toprule
Conflict Resolving Method & WS-L-G & SBM-L-G  \\
\midrule
Original Conflict Resolving & $\mathbf{41.1}$&$\mathbf{83.3}$ \\
Opposite Conflict Resolving & $40.7$ & $83.2$ \\
\bottomrule
\end{tabular}
\end{threeparttable}
\end{table}

\color{black}
\section{Details about Node Feature Encoder} \label{sec:dnf}
Motivated by previous approaches \citep{lorchamortized,ke2023learning}, we employ a transformer-like architecture comprising attention layers over either the observation dimension or the node dimension alternately as the node feature encoder.
Concretely, for the raw node features $\mathcal{F} \in \mathbb{R}^{d \times n \times h}$ corresponding to $d$ nodes and $n$ observations, our goal is to capture the correlations between both different nodes and different observations.
Therefore, we utilize two transformer encoder layers over the observation dimension and the node dimension alternatively:
\begin{align}
\begin{aligned}
    \mathcal{F} & \leftarrow TransformerEncoderLayer(\mathcal{F}, \mathcal{F}, \mathcal{F}) \\
    \mathcal{F} & \leftarrow \mathcal{F}.transpose(0, 1) \\
        \mathcal{F} & \leftarrow TransformerEncoderLayer(\mathcal{F}, \mathcal{F}, \mathcal{F})\\
    \mathcal{F} & \leftarrow \mathcal{F}.transpose(0, 1). \\
\end{aligned}
\end{align}
The above operation is repeated multiple times for sufficiently feature encoding.
It yields the final node feature tensor $\mathcal{F} \in \mathcal{R}^{d \times \times h}$.
\color{black}
\section{Illustration of the Case Study in Sec. \ref{sec:met:lim}}
Fig. \ref{fig:ps} presents an illustration for the case study of the Node-Edge approach in Sec. \ref{sec:met:lim}. 
It clearly shows that observational data with the two different parametrized forms follow the same joint distribution:
\begin{align}
    P(\left[X, Y, T\right]) =\mathcal{N}\left([0,0,0],\left[\begin{array}{lll}1 & 1 & 1 \\ 1 & 3 & 2 \\ 1 & 2 & 2\end{array}\right]\right).
\end{align}
Therefore, the observational datasets coming from the two DAGs are inherently indistinguishable.

\begin{figure*}[ht]
    \centering
    \includegraphics[width=0.9\linewidth]{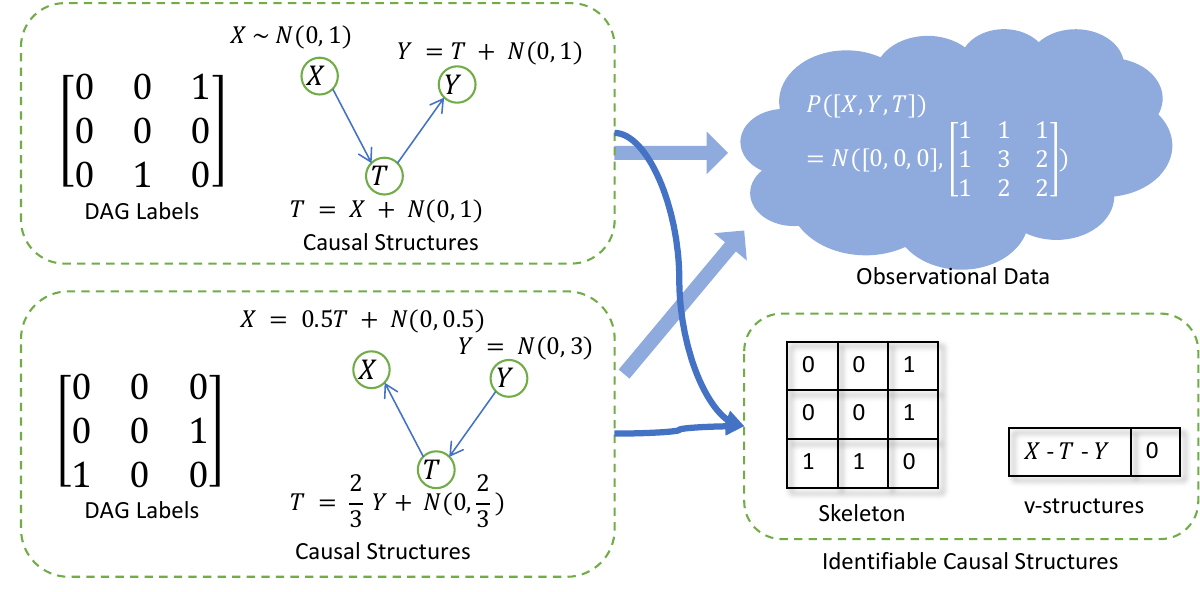}
    \caption{The problem setting to emphasize the limitations of the Node-Edge approach. \textit{Best viewed in color.}}
    \label{fig:ps}
\end{figure*}

\section{Proof and Discussion for Proposition \ref{prop:star}} \label{sec:mgcs}
\color{black}
We first restate the Proposition \ref{prop:star} with more details and provide the proof.
\begin{Proposition}
Let $\mathcal{G}_n$ be the set of graphs with $n+1$ nodes where there is a central node $y$ such that (1) every other node is connected to $y$, (2) there is no edge between the other nodes, (3) there is at most one edge pointing to $y$. 
For any distribution $Q$ over $\mathcal{G}_n$, let $M(Q)$ be another distribution over $\mathcal{G}_n$ such that for any causal edges $e, e'$, $P_{G\sim Q}(e \in G) = P_{G\sim M(Q)}(e \in G) = P_{G\sim M(Q)}(e \in G | e' \in G)$. We have 
\begin{align}\max_{Q} P_{G \sim M(Q)}(G \nin \mathcal{G}_n) = 
1 - \frac{2n-1}{n-1}(1 - \frac{1}{n})^n.
\end{align}
As a corollary, we have 
\begin{align}
\sup_n \max_{Q} P_{G \sim M(Q)}(G \nin \mathcal{G}_n) = 
1 - \frac{2}{e} \approx 0.2642,
\end{align}
\end{Proposition}
\begin{proof}
Denote other nodes except for the central node as $x_i$ where $i \in \left\{1, 2, \dots, n\right\}$. 
In our setting, the set $\mathcal{G}_n$ contains $n + 1$ DAGs
with the same skeleton and no v-structure
: $G_0: y \rightarrow x_i$ for all $x_i$, and $G_i: y \rightarrow x_j$ for all $x_j \neq x_i$ together with $x_i \rightarrow y$.
Denote the sampling probability of DAG $G_i$ from $\mathcal{G}_n$ as $P_i$.
Therefore, the marginal probability of the edge $y \rightarrow x_i$ is $1 - P_i$.


    If $\exists i$, $P_i = 1$, it means that $\mathcal{G}_n$ only contains the DAG $G_i$. Therefore, $M(Q)$ is equivalent to $Q$ and we have $P_{G \sim M(Q)}(G \nin \mathcal{G}_n) = 0$.

    If $\forall i$, $P_i < 1$, denoting $Q_i = 1 - P_i$ and $P(v)$ as the probability of $G$ containing no v-structures. In other words, $P(v) = P_{G \sim M(Q)}(G \in \mathcal{G}_n)$. We have 
    \begin{align}
        P(v) = \prod_{i=1}^n Q_i + \sum_{j=1}^n \frac{\prod_i^n Q_i}{Q_j} (1 - Q_j)
        &= ( \prod_{i=1}^n Q_i) \cdot (1 + \sum_{j=1}^{n} \frac{1-Q_j}{Q_j}).
    \end{align}
    As $P(v)$ is a probability, we have $P(v) > 0$. Denoting function 
    \begin{align}
    f(Q_1, Q_2, \dots, Q_n) = \log P(v) = \sum_{i=1}^n \log Q_i + \log (1 + \sum_{j=1}^n \frac{1-Q_j}{Q_j}),
    \end{align} we would like to find its minimum s.t. $\sum_i Q_i \geq n - 1$ and $Q_i \in (0, 1]$.

    Define its Lagrange function \begin{align}
        L(Q_1, Q_2, \dots, Q_n, \lambda) = f + \lambda (n-1-\sum_i Q_i).
    \end{align}

    We have 
    \begin{align}
        \frac{\partial L}{\partial \lambda} = n - 1 - \sum_i Q_i,
    \end{align}
    and 
    \begin{align}
        \frac{\partial L}{\partial Q_i} = \frac{1}{Q_i}(1 - \frac{1}{Q_i(1-n + \sum_{k=1}^{n}\frac{1}{Q_k})}) - \lambda.
    \end{align}

    Now we are going to find the extremums for $L(Q_1, Q_2, \dots, Q_n, \lambda)$.
    \begin{enumerate}
        \item[(1)] If $\lambda = 0$, we have $\forall i$, $\frac{\partial f}{\partial Q_i} = 0$, then
        \begin{align}
            \forall i, Q_i = \frac{1}{(1 - n + \sum_{k=1}^{n} \frac{1}{Q_k})}.
        \end{align}
        It indicates that $\forall i, Q_i = 1$, hence $f = 0$ and $P(v) = 1$.
        \item[(2)] If $\lambda \neq 0$, $\exists i$, we have $\forall i$, $\frac{\partial f}{\partial Q_i} = \lambda$ and $\sum_{i=1}^n = n - 1$. In other words, we have 
        \begin{align}
            \forall i, j, \frac{\partial f}{\partial Q_i} = \frac{\partial f}{\partial Q_j} = \lambda.
        \end{align}
        Define function 
        \begin{align}
        h(Q_i) = \frac{\partial f}{\partial Q_i} = \frac{1}{Q_i}(1 - \frac{1}{Q_i(1-n + \sum_{k=1}^{n}\frac{1}{Q_k})}).
        \end{align}
        we can rewrite the function as 
        \begin{align}
            h(Q_i) = \frac{1}{Q_i}(1 - \frac{1}{1 + AQ_i}),
        \end{align}
        where $A = 1 - n + \sum_{k\neq i} \frac{1}{Q_k} \geq 1 - n + \frac{(n-1)^2}{n-1-Q_i} > 0$.
        Therefore, $h(x)$ is a monotonic function in its domain. 

        It indicates that $\forall i, j$, $Q_i = Q_j = \frac{n-1}{n}$, where $P(v) = \frac{2n-1}{n-1}(1 - \frac{1}{n})^n$.
    \end{enumerate}
    Now we are going to list the boundary points for $f$.
    \begin{enumerate}
        \item[(1)] $\forall i$, $Q_i = 1$, it becomes the first extremum point.
        \item[(2)] $\exists i$, $Q_i$ is approaching to $0$. Due to the constraint of $\sum Q_i \geq n - 1$, other $Q$s are approaching to $1$. We have $\lim_{Q_i \rightarrow 0} f = 0$ and $P(v) = 1$.
    \end{enumerate}
    In conclusion, the maximum point of function $f$ is $\forall i$, $Q_i = \frac{n - 1}{n}$, where 
    \begin{align}
        P(v) = \frac{2n-1}{n-1}(1 - \frac{1}{n})^n,
    \end{align}
    and 
        \begin{align}
        P_{G \sim M(Q)}(G \nin \mathcal{G}_n) = 1- P(v) = 1 - \frac{2n-1}{n-1}(1 - \frac{1}{n})^n.
    \end{align}
\end{proof}
\textbf{Discussion. }It is worth noting that $\mathcal{G}_n$ is exactly the MEC of any graph in $\mathcal{G}_n$.
Hence, $P_{G \sim M(Q)}(G \nin \mathcal{G}_n)$ represents the probability that the graph sampled from $M(Q)$ is incorrect.
It indicates that a Node-Edge model could suffer from an inevitable error rate of $0.2642$ though has been perfectly trained to predict $M(Q)$.
\color{black}

\section{Experimental Settings} \label{sec:app:exp:set}
\color{black}
\textbf{Baselines.} To demonstrate the effectiveness and superiority of the proposed framework, 
several representative baselines from multiple categories are selected for comparison. 
The PC algorithm is a classic constraint-based causal discovery algorithm based on conditional independence tests, and the version with parallelized optimization is selected \citep{le2016fast}.
GES, a classic score-based greedy equivalence search algorithm, is also included \citep{chickering2002optimal}.
For continuous optimization methods, we compare with NOTEARS \citep{zheng2018dags}, a representative gradient-based optimization method, and GOLEM \citep{ng2020role}, regarded as a more efficient variant of NOTEARS. 
For neural network based optimization algorithms, we compare with DAG-GNN \citep{yu2019dag}, an optimization algorithm based on graph neural networks, and GRAN-DAG, a gradient-based algorithm using neural network modeling \citep{Lachapelle2020Gradient-Based}.
For DNN-based SCL methods, we compare with AVICI, which is the most related work to ours and regarded as the current state-of-the-art method \citep{lorchamortized}.
\color{black}

\textbf{Implementation Details. }The implementation from gCastle \citep{zhang2021gcastle} is utilized for baselines except the SCL methods (i.e., SLdisco and AVICI). 
For PC algorithm, we employ the Fisher-Z transformation with a significance threshold of $0.05$ for conditional independence tests, which is a prevalent choice in statistical analyses and current PC implementations \citep{zhang2021gcastle,zheng2024causal}.
Our criterion for graph selection in GES experiments is the Gaussian Bayesian Information Criterion (BIC), specifically the $l_\infty$-penalized Gaussian likelihood score. It is used in the original paper \citep{chickering2002optimal}, and remains a favored variant in the literature.
For NOTEARS, adhering to the official implementation's settings, we configure NOTEARS with a maximum of 100 dual ascent steps, and an edge dropping threshold of 0.3. For hyperparameters lacking specific default settings, such as the L1 penalty and loss function type, we default to settings used by gCastle \cite{zhang2021gcastle}, employing an L1 penalty of 0.1 and an L2 loss function.
For DAG-GNN, we utilize hyperparameter settings directly from the original implementation, ensuring consistency with established benchmarks.
For GOLEM and GRAN-DAG, we also use the default setting of gCastle \citep{zhang2021gcastle}.
\color{black}
Note that the CSIvA model \citep{ke2023learning} is also a closely related method, but it is not compared due to the unavailability of its relevant codes and its requirement for interventional data as input. 
The original implementation of SLdisco \cite{petersen2023causal} was developed in R. To enhance compatibility with our data generation and evaluation workflows, we reimplemented the model using PyTorch.
The original AVICI model \citep{lorchamortized} does not support discrete data.
Therefore, we use an embedding layer to replace its first linear layer when using AVICI on discrete data.

\textbf{Synthetic Data.} We randomly generate random graphs from multiple random graph models. For continuous data, following previous work \citep{lorchamortized}, Erdős-Rényi (ER) and Scale-free (SF) are utilized as the training graph distribution $p(G)$.
The degree of training graphs in our experiments varies randomly among 1, 2, and 3.
\textcolor{black}{For testing graph distributions, Watts-Strogatz (WS) and Stochastic Block Model (SBM) are used, with parameters consistent with those in the previous paper \citep{lorchamortized}. }
All synthetic graphs for continuous data contain 30 nodes.
\textcolor{black}{The lattice dimension of Watts-Strogatz (WS) graphs is sampled from $\{2, 3\}$, yielding an average degree of about $4.92$. The average degrees of Stochastic Block Model (SBM) graphs are set at 2, following the settings in the aforementioned paper.}
For discrete data, 11-node graphs are used.
SF is utilized as the training graph distribution $p(G)$ and ER is used for testing.
\textcolor{black}{The synthetic training data is generated in real-time, and the training process does not use the same data repeatedly.}
All synthetic test datasets contain 100 graphs, and the average values of the metrics on the 100 graphs are reported to comprehensively reflect the performance.

For the forward sampling process from graph to continuous data, both the linear Gaussian mechanism and general nonlinear mechanism are applied.
Concretely, the Random Fourier Function mechanism is used for the general nonlinear data following the previous paper \citep{lorchamortized}.
In synthesizing discrete datasets, the Bernoulli distribution is used following previous papers \citep{dai2023ml4c,ma2022ml4s}.

\begin{figure*}
    \centering
    \includegraphics[width=\linewidth]{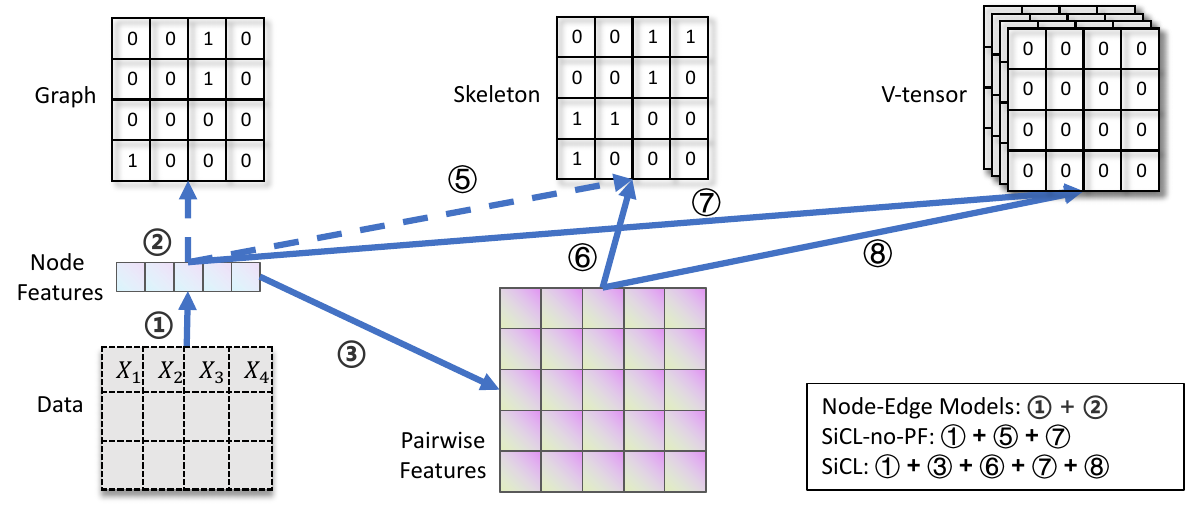}
    \caption{Illustration of the architecture comparison of Node-Edge models, SiCL-no-PF and SiCL.}
    \label{fig:abl}
\end{figure*}

\textbf{More Implementation Details and Computational Resources. } The two network modules, i.e., the SPN and VPN, are optimized by Adam optimizer with default hyperparameters. Following previous work \citep{lorchamortized}, the training batch size is set as 20. 
All classic algorithms are run on an AMD EPYC 7V13 CPU, and DNN-based methods are run on Nvidia 1080Ti, A40 and A100 GPUs. 
Training SiCL on a 30-node training set with batch size 20 needs about 60GB memory, and training on a 11-node training set needs about 20GB memory.
The learning rate is $3 \times 10^{-4}$.
The batch size is $15$, and the DNN models are trained for $1.2 \times 10^{5}$ batches by default.
\color{black}

\begin{table*}[t]
\centering
\begin{threeparttable}
\caption{\textbf{General comparison of SiCL and other methods}. The average performance results in three runs are provided for SiCL method. GES takes more than 24 hours per graph on WS-L-G, and SLdicso is
unsuitable on non-linear-Gaussian data, hence the results are not included.}
\label{tab:epder}
\begin{tabular}{@{}ccccccccc@{}}
\toprule
\multirow{2}{*}{Dataset} &\multirow{2}{*}{Method} & \multicolumn{4}{c}{Skeleton Prediction }& \multicolumn{3}{c}{CPDAG Prediction} \\
 && s-F1$\uparrow$ & s-Acc.$\uparrow$ & s-AUC$\uparrow$ & s-AUPRC$\uparrow$ &  v-F1$\uparrow$ & o-F1$\uparrow$ & SHD$\downarrow$\\
 \midrule
\multirow{8}{*}{WS-L-G}&PC & $30.4 $&$ 65.6$&N/A&N/A&$ 15.6$&$16.0 $ &$170.4 $\\
&NOTEARS & $33.3 $  & $ 65.1$&N/A&N/A&$ 27.9$&$31.5$ & $159.8$ \\
&DAG-GNN & $35.5$  & $ 55.4$&N/A&N/A&$32.2$&$32.7$ & $193.7 $ \\
&GRAN-DAG& $16.6$ & $62.1$ & N/A & N/A & $11.7$ & $11.7$ & $170.1$ \\
&GOLEM& $30.0$ & $63.4$ & N/A & N/A & $15.8$ & $19.3$ & $172.7$ \\
&SLdisco & $0.1$ & $66.0$ & $50.2$ & $34.6$ & $0.0$ & $0.1$ & $147.9$ \\
&AVICI & $39.9 $  & $ 74.0$ &$ 71.5$&$ 62.2$&$ 28.2$&$ 35.8$&$119.2$\\
&SiCL & $\mathbf{44.7} $ & $\mathbf{75.3} $&$\mathbf{73.7} $&$\mathbf{65.4} $&$\mathbf{32.0} $&$\mathbf{38.5} $&$\mathbf{116.1} $\\
 \midrule
\multirow{9}{*}{SBM-L-G}&PC & $58.8$&$90.0$&N/A&N/A &$34.8$&$35.9$&$56.4$\\
&GES & $70.8$&$89.4$ &N/A&N/A&$53.9$&$55.0$&$60.3$\\
&NOTEARS & $80.1$  & $94.5$&N/A&N/A&$76.2$&$77.8$ & $26.7$ \\
&DAG-GNN & $66.2$  & $87.4$&N/A&N/A&$60.3$&$62.5$ & $61.0$ \\
&GRAN-GAG & $22.6$  & $85.9$&N/A&N/A&$13.8$&$14.4$ & $64.7$ \\
&GOLEM & $68.5$ &$88.5$ &N/A&N/A&$63.5$&$65.2$&$55.1$\\
&SLdisco & $1.9$ & $85.7$ & $56.3$ & $17.6$ & $0.9$ & $1.2$ & $62.6$ \\
&AVICI & $84.3$  & $96.2$ &$98.1$&$92.7$&$79.1$&$81.6$&$17.7$\\
&SiCL &  $\mathbf{85.8}  $  & $\mathbf{96.4}  $& $\mathbf{98.3}$& $\mathbf{93.4} $&$\mathbf{80.6}  $&$\mathbf{82.7}  $&$\mathbf{17.1}  $\\
\midrule
\multirow{9}{*}{WS-RFF-G}&PC & $36.1 $&$69.9$&N/A&N/A &$ 14.8$&$16.1$&$ 156.9$\\
&GES & $ 41.7$&$66.6$ &N/A&N/A&$21.1 $&$23.6$&$174.1$\\
&NOTEARS & $37.7 $  & $64.6 $& N/A & N/A &$30.9 $&$33.4 $ & $164.4 $ \\
&DAG-GNN & $33.2 $  & $ 65.4$&N/A&N/A&$27.0 $&$28.9 $ & $161.1 $ \\
&GRAN-DAG & $4.7 $  & $ 66.7$&N/A&N/A&$0.8 $&$1.1 $ & $146.9 $ \\
 &GOLEM & $27.6$ &$62.4$ &N/A&N/A&$13.8$&$17.7$&$175.8$\\
&AVICI & $47.7 $  & $75.9$ &$ 76.3$&$ 67.6$&$38.7$&$45.2 $&$110.6 $\\
&SiCL & $\mathbf{51.8}  $  & $ \mathbf{77.4}$ &$\mathbf{81.1}$&$ \mathbf{72.9}$&$ \mathbf{40.3}$&$ \mathbf{46.3}$&$ \mathbf{107.0} $\\
\midrule
\multirow{9}{*}{SBM-RFF-G}&PC & $57.5$&$89.3$&N/A&N/A &$32.7$&$ 34.2$&$60.9$\\
&GES & $56.5 $&$84.9$ &N/A&N/A&$37.0$&$38.0$&$82.4$\\
&NOTEARS & $55.6 $  & $86.2 $&N/A&N/A&$ 46.5$&$ 48.5$ & $66.3 $ \\
&DAG-GNN & $ 47.1$  & $82.1 $&N/A&N/A&$39.0$&$40.6$ & $86.2 $ \\
&GRAN-DAG & $17.4$ &$87.4$ & N/A& N/A &$3.2 $&$3.8$&$58.2$\\
&GOLEM & $31.1$ &$75.7$ & N/A& N/A &$23.0 $&$24.8$&$112.0$\\
&AVICI & $76.6$  & $ 94.5$ &$ 95.4$&$85.7 $&$69.3$&$72.7 $&$ 27.2$\\
&SiCL & $ \mathbf{82.1}$  & $ \mathbf{95.7}$ &$ \mathbf{97.1}$&$ \mathbf{90.7} $&$ \mathbf{75.7} $&$ \mathbf{78.0}$&$ \mathbf{21.9} $\\
\midrule
\multirow{8}{*}{ER-CPT-MC}&PC & $82.2$&$83.0$ &N/A&N/A&$39.2$&$40.6$&$16.4$\\
&GES & $ 82.1$&$81.8$ &N/A&N/A&$40.4$&$42.4$&$17.1$\\
&NOTEARS & $16.7$  & $74.8$&N/A&N/A&$0.2$&$0.6$& $16.1$ \\
&DAG-GNN & $ 24.8$  &  $73.5$&N/A&N/A&$ 3.4$&$3.7 $ & $15.9 $ \\
&GRAN-DAG & $40.8$ &$77.0$ & N/A& N/A &$6.8 $&$7.3$&$15.6$\\
&GOLEM& $37.6$ & $66.4$ & N/A & N/A & $4.6$ & $9.3$ & $21.9$ \\
&AVICI & $76.9$  & $88.4$ & $93.5 $& $87.9 $&$56.6$&$57.6$&$10.2$\\
&SiCL &  $\mathbf{84.2}$  & $\mathbf{90.1}$&$ \mathbf{96.6}$& $\mathbf{94.0} $& $ \mathbf{58.3}$&$\mathbf{59.9}$&$\mathbf{10.1}$\\
\bottomrule
\end{tabular}
\end{threeparttable}
\end{table*}

\begin{table*}[tb]
    \centering
\begin{threeparttable}
\caption{Full ablation study results.}
\label{tab:fcplg}
\begin{tabular}{ccccccccc}
\toprule
Dataset &Method & s-F1$\uparrow$ & s-Acc.$\uparrow$ & s-AUC$\uparrow$ & s-AUPRC$\uparrow$ &  v-F1$\uparrow$ & o-F1$\uparrow$ & SHD$\downarrow$\\
 \midrule
\multirow{3}{*}{WS-L-G}& SiCL-Node-Edge & $39.9 $&$ 74.0$&$71.5$&$62.2$&$ 28.2$&$35.8 $ &$119.2 $\\
&SiCL-no-PF & $ 42.4$  & $ 74.4$ &$ 72.8$&$ 63.5$&$ 30.5$&$ 37.9$ &$118.4 $\\
&SiCL & $\mathbf{44.7} $ & $\mathbf{75.3} $&$\mathbf{73.7} $&$\mathbf{65.4} $&$\mathbf{32.0} $&$\mathbf{38.5} $&$\mathbf{116.1} $\\ \hline
\multirow{3}{*}{SBM-L-G}& SiCL-Node-Edge & $ 84.3$&$ 96.2$&$98.1$&$92.7$&$ 79.1$&$81.6 $ &$17.7 $\\
&SiCL-No-PF & $85.5$  & $\mathbf{96.4}$ &$\mathbf{98.3}$&$93.3$&$79.4$&$82.2$&$17.3$\\
&SiCL &  $\mathbf{85.8}$  & $\mathbf{96.4} $& $\mathbf{98.3} $& $\mathbf{93.4} $&$\mathbf{80.6}  $&$\mathbf{82.7}  $&$\mathbf{17.1}$\\
\bottomrule
\end{tabular}
\end{threeparttable}
\end{table*}
\section{Extra Experimental Results} \label{sec:app:exp:e}

\begin{figure*}[t]
     \centering
     \begin{subfigure}[b]{0.45\textwidth}
         \centering
         \includegraphics[width=\textwidth]{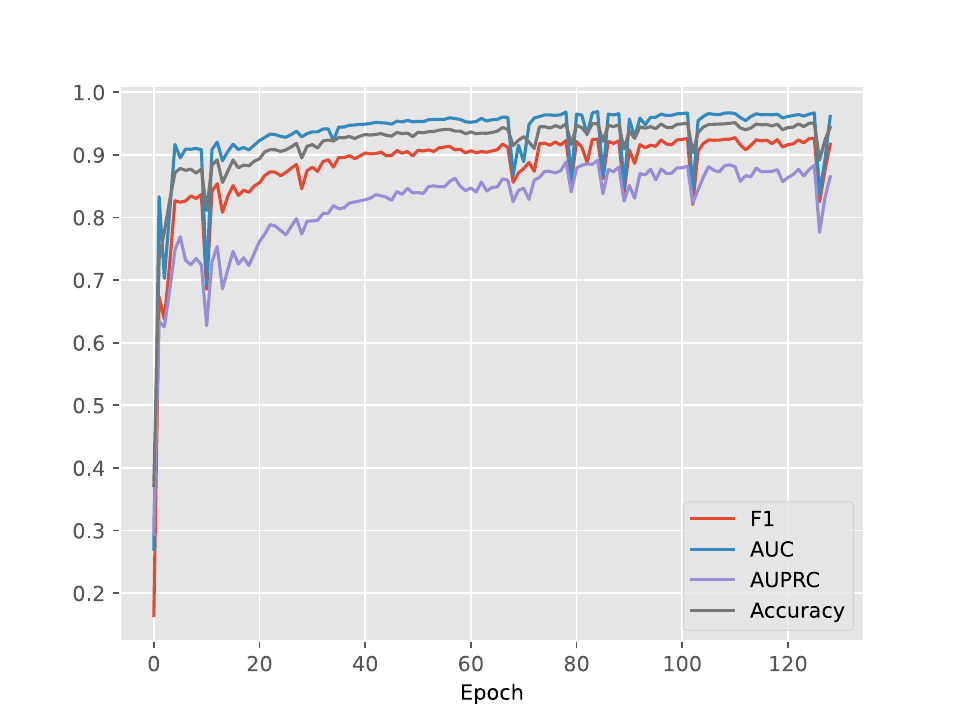}
         \caption{WS-LG}
         \label{fig:vws}
     \end{subfigure}
     \hfill
     \begin{subfigure}[b]{0.45\textwidth}
         \centering
         \includegraphics[width=\textwidth]{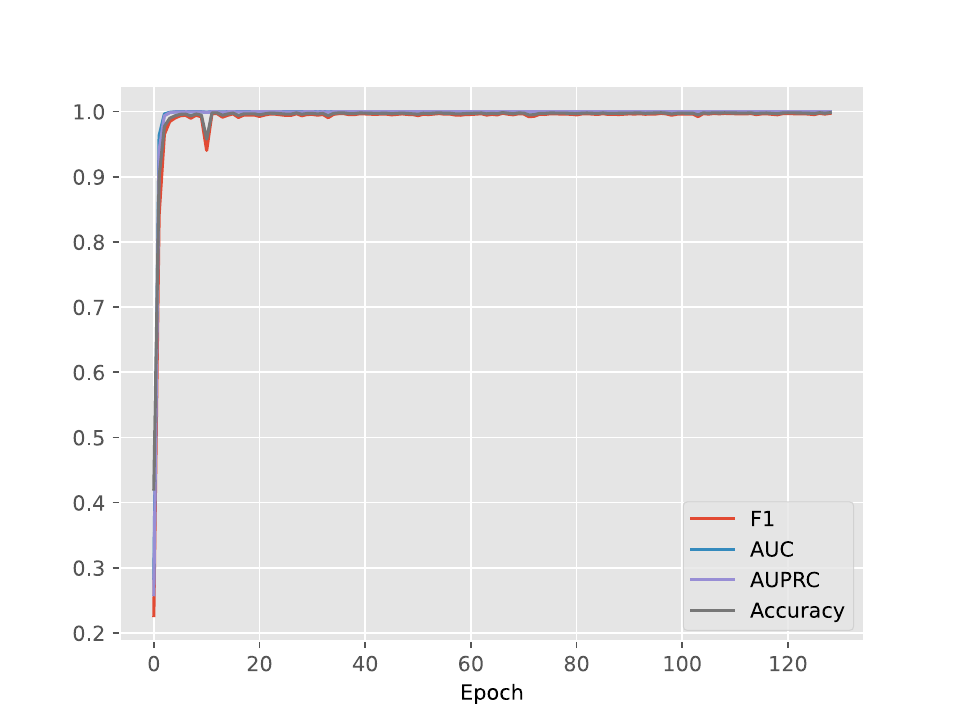}
         \caption{SBM-LG}
         \label{fig:vsbm}
     \end{subfigure}
     \caption{Variation trends of the test performance of the V-structure Prediction Network on WS-LG and SBM-LG during training.}
        \label{fig:tcs}
\end{figure*}

\subsection{Effectiveness of V-structure Prediction Network} Fig. \ref{fig:tcs} illustrates the test performance trends of the v-structure prediction model on SBM and WS random graphs during the training process. In this model, the feature extractor $FE$ is fine-tuned from the skeleton prediction model. The performance increases rapidly and achieves a relatively high level after just a few initial epochs. This suggests that our v-structure prediction network is capable to predict v-structures, and indicates that the pre-trained pairwise features from the skeleton prediction model are both effective and generalizable.

\color{black}
\subsection{More Evidence on Effectiveness of Pairwise Representation}
To further support the effectiveness of using pairwise representation, we present additional experimental results on different training datasets and test datasets, including ER-L-G, SF-L-G, ER-RFF-G, and SF-RFF-G.
For models, we compare SiCL with a variant without pairwise representation, i.e., SiCL-no-PF.

The results are provided in Tab. \ref{tab:mpc}. Models with pairwise representation ourperform the corresponding baseline models under almost all comparisons, further verifying the effectiveness of using pairwise representation in models.

\begin{table}[t]\color{black}
\centering
\begin{threeparttable}
\caption{\color{black}More performance comparison on the effectiveness of pairwise representation.}\label{sec:mpc}
\label{tab:mpc}
\begin{tabular}{ccccccc}
\toprule 
Training Dataset & Test Dataset & Method & s-F1$\uparrow$ & s-AUC$\uparrow$ & s-AUPRC$\uparrow$ & s-Acc.$\uparrow$ \\
\midrule
\multirow{8}{*}{ER-L-G} & \multirow{2}{*}{ER-L-G} & SiCL-no-PF & 75.7 &84.6  & 83.1 & 78.5\\
 &  & SiCL &  80.2 & 89.6 & 90.1 &82.3 \\ \cline{2-7}
 & \multirow{2}{*}{SF-L-G} & SiCL-no-PF & 74.9 &  92.5& 87.3 &84.1 \\
 &  & SiCL & 79.0 & 96.0 & 93.7 & 87.0\\ \cline{2-7}
  & \multirow{2}{*}{ER-RFF-G} & SiCL-no-PF & 49.5 & 60.5 & 49.1 &58.8 \\
 &  & SiCL & 51.0 & 67.0 & 57.6 & 65.2\\ \cline{2-7}
  & \multirow{2}{*}{SF-RFF-G} & SiCL-no-PF & 40.4 & 57.9 & 38.9 & 57.5\\
 &  & SiCL & 46.4 & 71.4 & 53.7 & 69.0\\ \hline
 \multirow{8}{*}{SF-L-G} & \multirow{2}{*}{ER-L-G} & SiCL-no-PF & 64.6 & 77.3 & 68.7 & 70.7\\
 &  & SiCL & 68.0 & 82.1 & 76.4 & 74.3\\ \cline{2-7}
 & \multirow{2}{*}{SF-L-G} & SiCL-no-PF & 88.5 & 96.7 & 95.0 & 91.2\\
 &  & SiCL & 89.7 & 97.9 & 97.0 & 92.4\\ \cline{2-7}
  & \multirow{2}{*}{ER-RFF-G} & SiCL-no-PF & 44.3 & 62.3 & 50.9 & 58.4\\
 &  & SiCL & 47.0 & 66.2 & 55.8 & 63.3\\ \cline{2-7}
  & \multirow{2}{*}{SF-RFF-G} & SiCL-no-PF & 48.1 & 71.6 & 53.9 & 65.8 \\
 &  & SiCL & 56.0 & 79.6 & 64.8 & 74.2\\ \hline
  \multirow{8}{*}{ER-RFF-G} & \multirow{2}{*}{ER-L-G} & SiCL-no-PF & 64.0 & 73.3 & 65.8 & 67.3 \\
 &  & SiCL & 72.0 & 82.0 & 81.1 & 75.2 \\ \cline{2-7}
 & \multirow{2}{*}{SF-L-G} & SiCL-no-PF & 58.1 & 79.0 & 66.8 & 72.8\\
 &  & SiCL & 70.1 & 88.0 & 83.6 & 80.9\\ \cline{2-7}
  & \multirow{2}{*}{ER-RFF-G} & SiCL-no-PF & 63.2 & 74.3 & 67.7 & 71.0\\
 &  & SiCL & 74.8 & 85.7 & 84.5 & 79.7\\ \cline{2-7}
  & \multirow{2}{*}{SF-RFF-G} & SiCL-no-PF & 56.3 & 78.3 & 65.9 & 75.0 \\
 &  & SiCL & 68.2 & 87.0 & 81.5 & 82.1 \\ \hline
  \multirow{8}{*}{SF-RFF-G} & \multirow{2}{*}{ER-L-G} & SiCL-no-PF & 60.3 &  71.2 & 58.7 & 64.7\\
 &  & SiCL & 65.6 & 78.0 & 72.5 & 70.5\\ \cline{2-7}
 & \multirow{2}{*}{SF-L-G} & SiCL-no-PF & 73.6 & 90.5 & 82.9 & 81.0\\
 &  & SiCL & 79.1 & 94.2 & 90.0 & 85.5\\ \cline{2-7}
  & \multirow{2}{*}{ER-RFF-G} & SiCL-no-PF & 57.7 & 71.4 & 60.7 & 66.8\\
 &  & SiCL & 67.2 & 80.5 & 75.8 & 73.9\\ \cline{2-7}
  & \multirow{2}{*}{SF-RFF-G} & SiCL-no-PF & 74.8 & 90.2 & 82.4 & 83.5\\
 &  & SiCL & 80.4 & 94.2 & 90.5 & 87.3\\ 
\bottomrule 
\end{tabular}
\end{threeparttable}
\end{table}

\subsection{Additional Comparison on DAG Prediction}
We provide an additional comparison with the AVICI baseline on the DAG prediction task. Since SiCL predicts CPDAGs and does not directly produce DAG predictions, we corrected the DAG predictions from AVICI using the edge directions inferred from the CPDAGs predicted by SiCL. The results, summarized in the Table \ref{tab:acdp}, demonstrate that incorporating CPDAG-inferred edge directions improves the DAG prediction metrics. This further confirms the effectiveness and generality of our approach, even in tasks focused on DAG metrics. 
\begin{table}[]
    \centering
        \caption{Additional Comparison on DAG Prediction}
    \label{tab:acdp}
    \begin{threeparttable}
    \begin{tabular}{cccccc}
         \toprule  
         Method & Dataset &F1 Score$\uparrow$ & AUC$\uparrow$ & AUPRC$\uparrow$ & Acc.$\uparrow$ \\ \midrule 
         AVICI & \multirow{2}{*}{WS-L-G} &$38.4$ &$86.3$ &$57.7$ &$85.9$ \\
         SiCL-Corrected AVICI && $35.8$&$87.2$ &$60.5$ &$86.2$  \\
         AVICI & \multirow{2}{*}{SBM-L-G}& $78.1$& $95.8$&$80.5$ &$97.3$ \\
         SiCL-Corrected AVICI & &$81.3$&$98.7$ &$90.8$ & $97.8$ \\
         \bottomrule
    \end{tabular}
        \end{threeparttable}
\end{table}

\subsection{Comparison with Autoregressive models on Inference Time Costs} \label{sec:auto}
To validate that the autoregressive models have a relatively high time costs due to the quadratic number of inference runs w.r.t. number of variables, we reproduce the network architecture of a representative autoregressive model, i.e., CSIvA \citep{ke2023learning}, and compare SiCL with it.
We use the same random input for both the models with increasing number of variables.
The results are provided in Fig. \ref{fig:itc}.
The time costs of the autoregressive model show a fast increasing trend and are much more than costs of SiCL, validating the correctness of our analysis.

\begin{figure}
    \centering
    \includegraphics[width=0.5\linewidth]{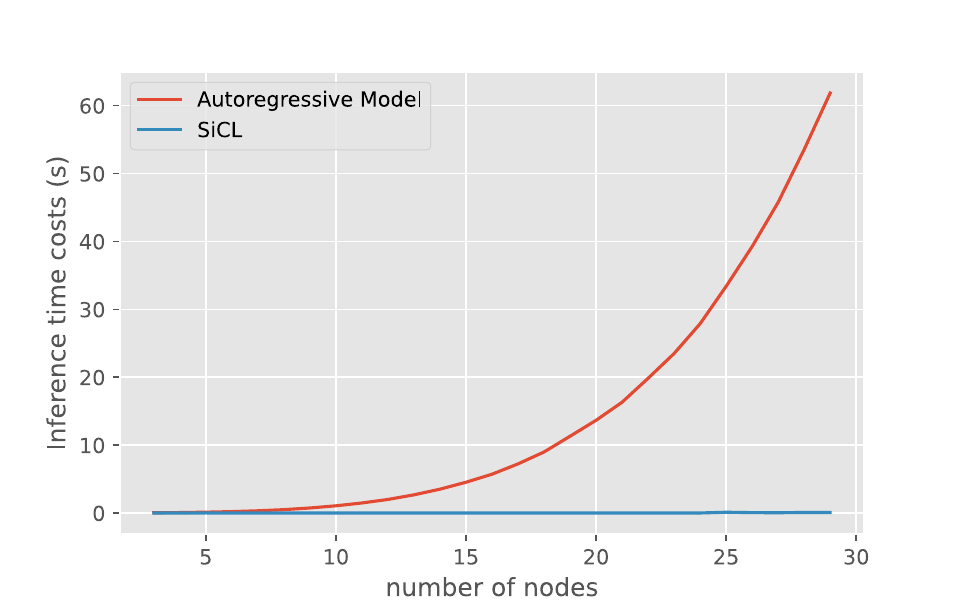}
    \caption{Comparison between an autoregressive model and SiCL on inference time costs.}
    \label{fig:itc}
\end{figure}

\color{black}
\subsection{Training Data Diversity and Model Generalization} We present experimental evidence that highlights the significant contribution of training data diversity to the model's generalization capabilities, even when applied to out-of-distribution (OOD) datasets. 
To illustrate this, we train one SiCL model on a combined dataset of both SF and ER, and another solely on the SF dataset. 
The comparative performance of these models is detailed in Tab. \ref{tab:trainood}. 
The model trained on the combined ER and SF datasets exhibited markedly better performance, not only on the ER dataset but also on the other two OOD datasets, with only a marginal decrease in performance on the SF dataset. 
These findings suggest that enhancing the diversity of the training data correspondingly improves the model’s ability to generalize and maintain robust performance across novel OOD datasets.

\begin{table*}[tb]\color{black}
    \centering
    \caption{\color{black}Comparison of SiCL models with different training data diversity on skeleton prediction.}
    \label{tab:trainood}
    \begin{subtable}{\linewidth}
      \centering
        \caption{\color{black}Model trained on both ER and SF}
        \begin{tabular}{lcccc}
            \toprule
            Test Dataset & s-F1$\uparrow$     & s-AUC$\uparrow$    & s-AUPRC$\uparrow$  & s-Acc.$\uparrow$    \\
            \midrule
            WS-L-G      & 36.3 & 70.6 & 60.6 & 73.3 \\
            SBM-L-G     & 78.1 & 96.8 & 88.1 & 94.8 \\
            ER-L-G      & 80.7 & 96.0 & 89.2 & 94.7 \\
            SF-L-G      & 84.7 & 98.5 & 93.6 & 95.5 \\
            \bottomrule
        \end{tabular}
    \end{subtable}%
    \\
    \begin{subtable}{\linewidth}
      \centering
        \caption{\color{black}Model trained on SF}
        \begin{tabular}{lcccc}
            \toprule
            Test Dataset & s-F1$\uparrow$     & s-AUC$\uparrow$    & s-AUPRC$\uparrow$  & s-Acc.$\uparrow$    \\
            \midrule
            WS-L-G      & 40.1 & 63.0 & 46.1 & 63.5 \\
            SBM-L-G     & 64.3 & 91.7 & 72.9 & 90.9 \\
            ER-L-G      & 67.1 & 90.4 & 73.9 & 90.8 \\
            SF-L-G      & 87.8 & 98.9 & 95.3 & 96.1 \\
            \bottomrule
        \end{tabular}
    \end{subtable}
\end{table*}

\subsection{Varying Amount of Training Graphs}
We present an analysis of how varying the amount of the training graphs influences performance on the skeleton prediction task. The results, depicted in Fig. \ref{fig:ts}, illustrate a clear trend: model performance improves in tandem with the expansion of the training dataset. This trend underscores the potential of our method to achieve even greater accuracy given a more extensive dataset.
\begin{figure*}[!ht]
     \centering
     \begin{subfigure}[b]{0.45\textwidth}
         \centering
    \includegraphics[width=\linewidth]{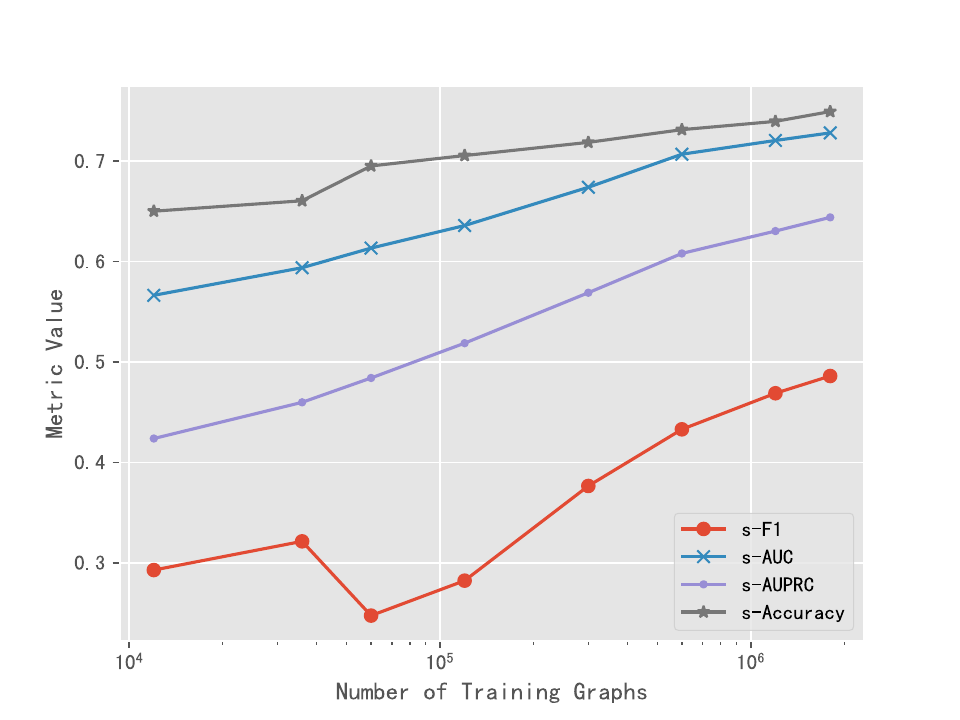}
         \caption{\color{black}WS dataset}
         \label{fig:ts1}
     \end{subfigure}
     \hfill
     \begin{subfigure}[b]{0.45\textwidth}
         \centering
    \includegraphics[width=\linewidth]{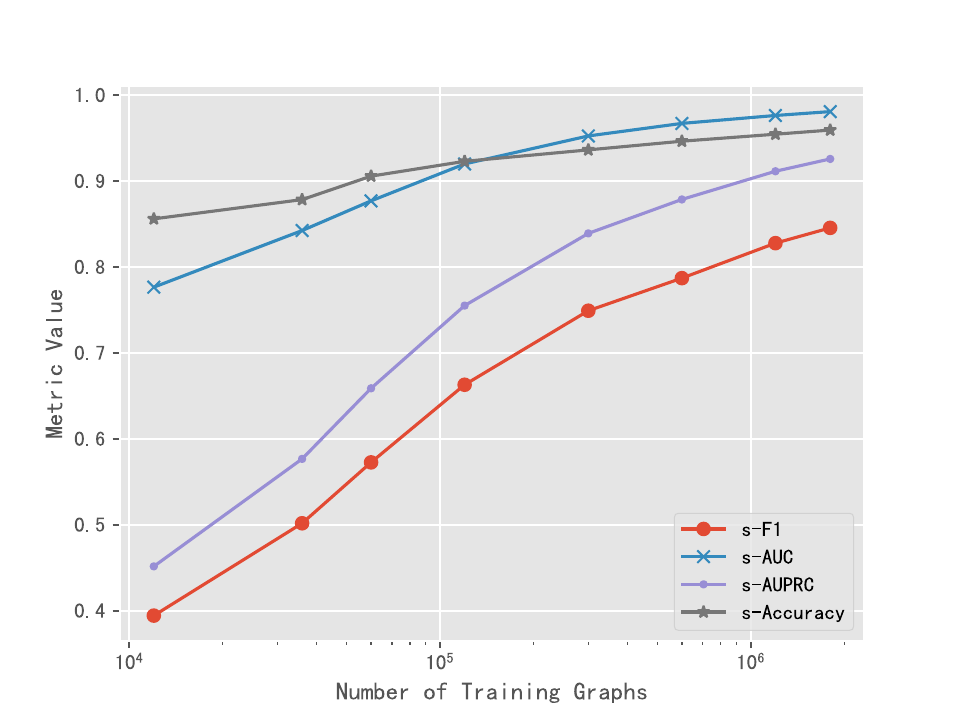}
         \caption{\color{black}SBM dataset}
         \label{fig:ts2}
     \end{subfigure}
         \caption{\color{black}Model performance with varying amount of training graphs.}
        \label{fig:ts}
\end{figure*}

\subsection{Varying Sample Size} We assess SiCL across various quantities of observational samples per graph during testing (100, 200, ..., 1000). The outcomes for both the skeleton prediction task and the CPDAG prediction task are depicted in Fig. \ref{fig:vtss}. It is evident that the model's performance enhances with the augmentation of sample size. These consistent upward trends suggest that SiCL exhibits stability and is not overly sensitive to changes in sample size.

\begin{figure*}[t]
     \centering
     \begin{subfigure}[b]{0.45\textwidth}
         \centering
         \includegraphics[width=\textwidth]{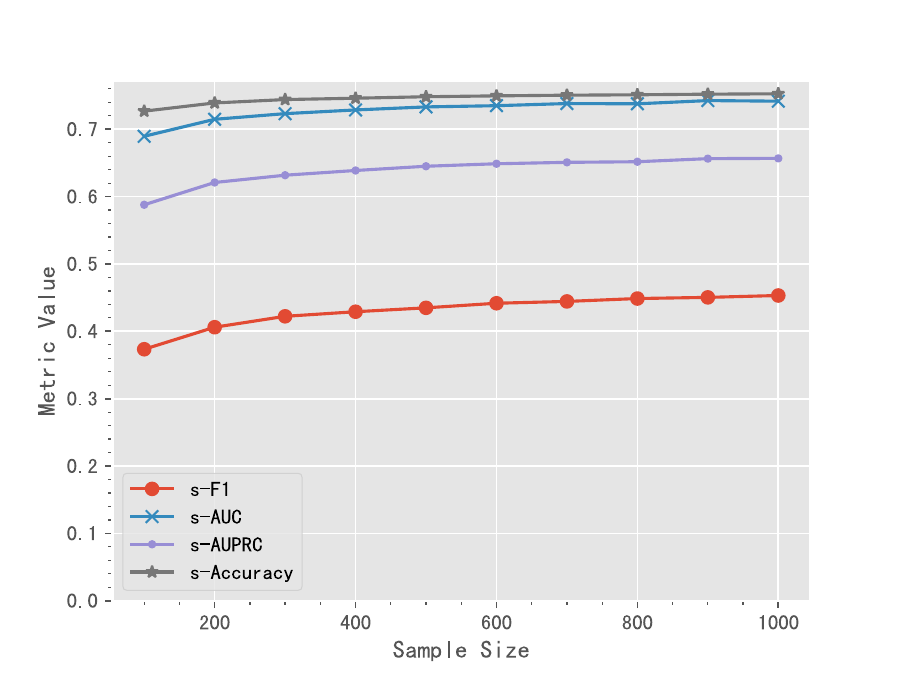}
         \caption{\color{black}Variation trends of skeleton predicton task performance on WS graph with varying sample sizes.}
         \label{fig:vtss1}
     \end{subfigure}
     \hfill
     \begin{subfigure}[b]{0.45\textwidth}
         \centering
         \includegraphics[width=\textwidth]{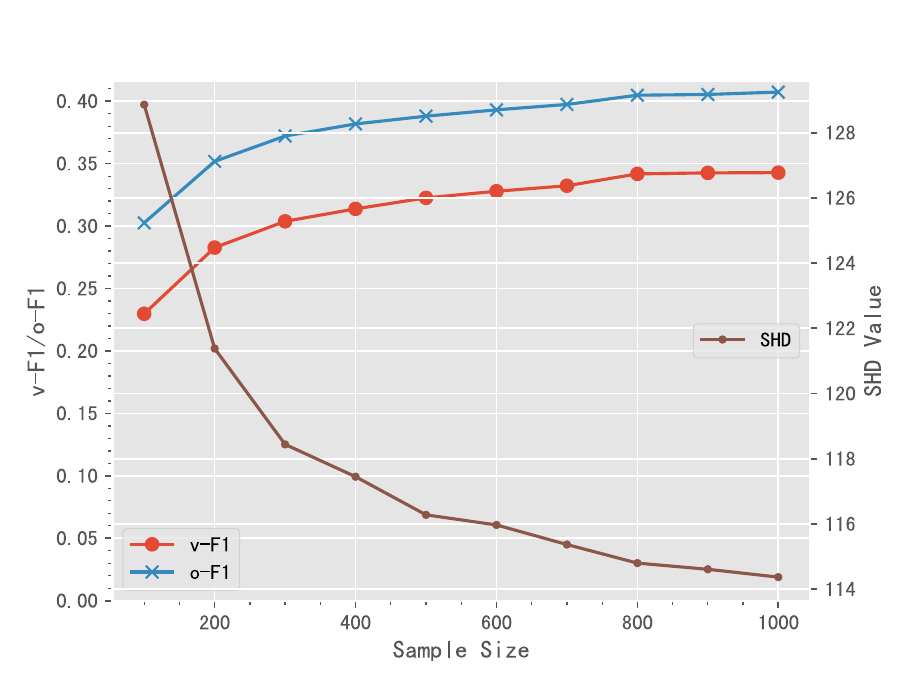}
         \caption{\color{black}Variation trends of CPDAG predicton task performance on WS graph with varying sample sizes.}
         \label{fig:vtss2}
     \end{subfigure}
         \begin{subfigure}[b]{0.45\textwidth}
         \centering
         \includegraphics[width=\textwidth]{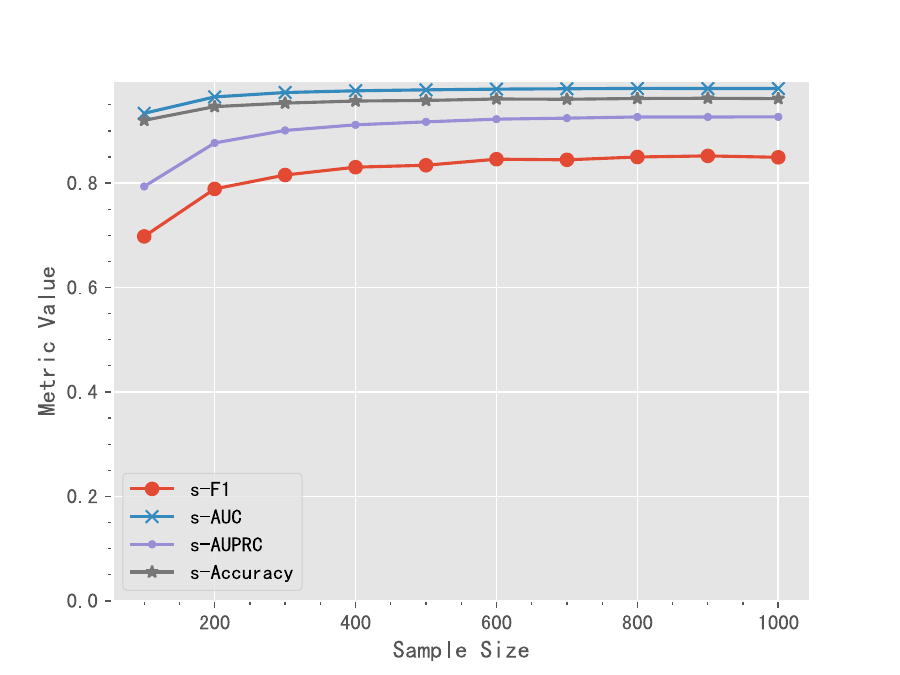}
         \caption{\color{black}Variation trends of skeleton predicton task performance on SBM graph with varying sample sizes.}
         \label{fig:vtss3}
     \end{subfigure}
     \hfill
     \begin{subfigure}[b]{0.45\textwidth}
         \centering
         \includegraphics[width=\textwidth]{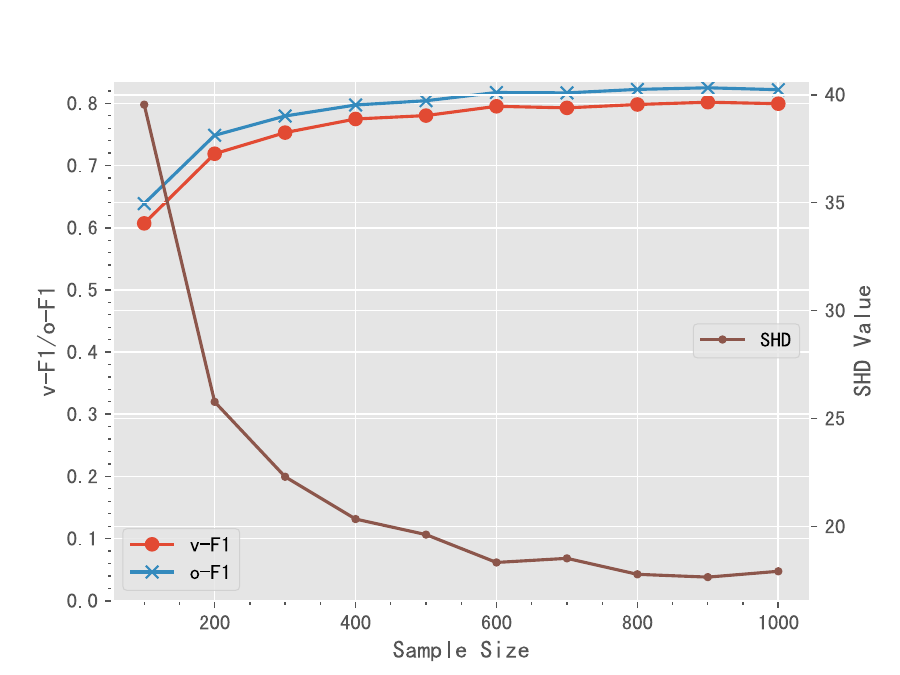}
         \caption{\color{black}Variation trends of CPDAG predicton task performance on SBM graph with varying sample sizes.}
         \label{fig:vtss4}
     \end{subfigure}
         \caption{\color{black}Variation trends of performance with varying sample sizes.}
        \label{fig:vtss}
\end{figure*}

\subsection{Varying Edge Density}
We evaluate SiCL over a range of edge densities in the test graphs, utilizing the SBM dataset, as it allows for the direct setting of average edge densities. The findings are presented in Fig. \ref{fig:vted}. It's apparent that the task is becomes more difficult as edge densities increase. However, the performance decline is not abrupt, indicating that SiCL's performance remains relatively stable across various edge densities, thereby confirming its versatility.
\begin{figure*}
     \centering
     \begin{subfigure}[b]{0.45\textwidth}
         \centering
         \includegraphics[width=\textwidth]{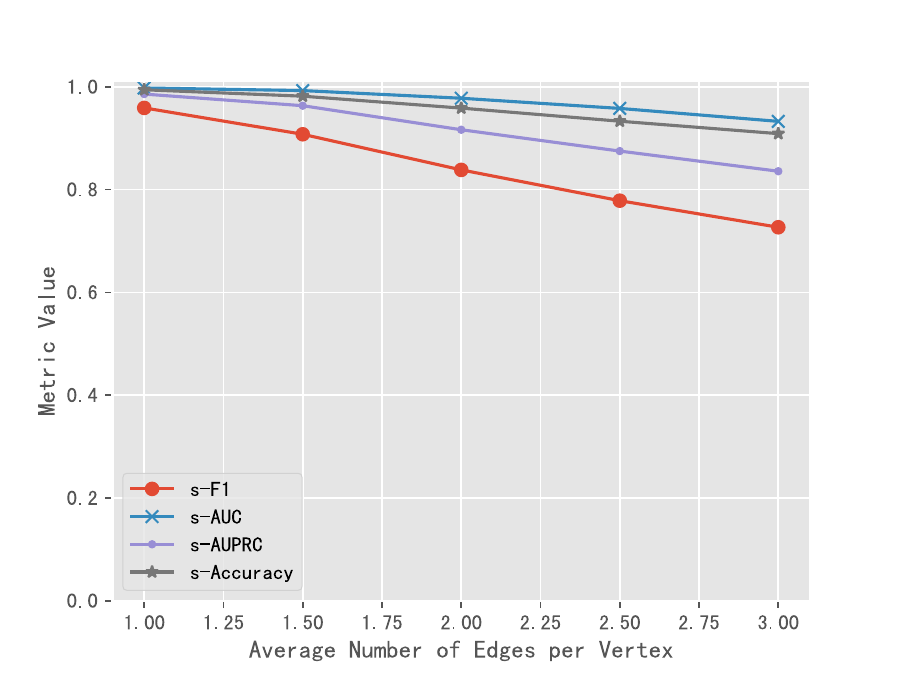}
         \caption{\color{black}Variation trends of skeleton predicton task performance on SBM graph with varying edge densities.}
         \label{fig:vted1}
     \end{subfigure}
     \hfill
     \begin{subfigure}[b]{0.45\textwidth}
         \centering
         \includegraphics[width=\textwidth]{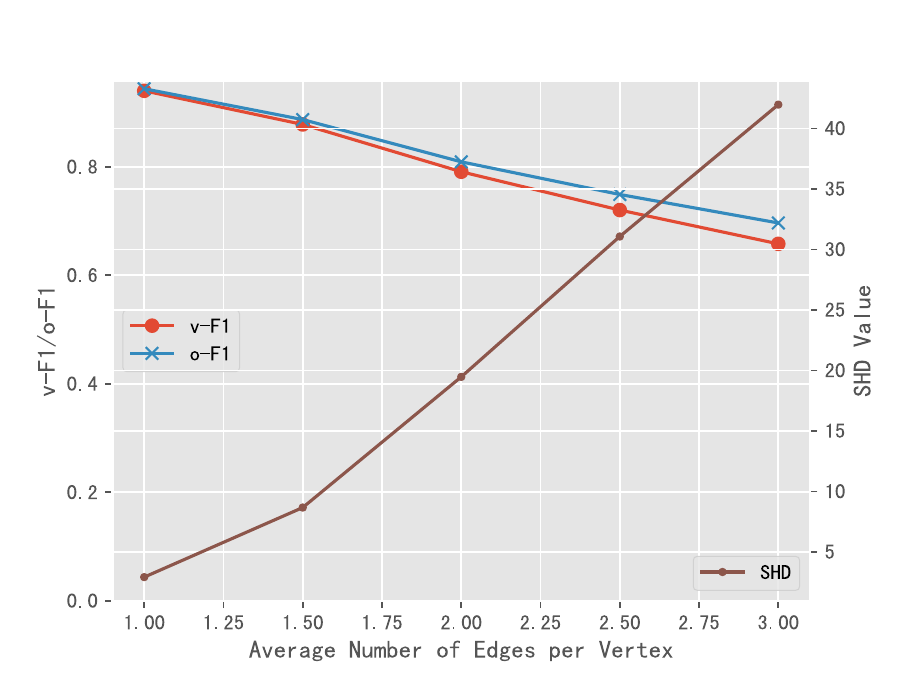}
         \caption{\color{black}Variation trends of CPDAG predicton task performance on SBM graph with varying edge densities.}
         \label{fig:vted2}
     \end{subfigure}
         \caption{\color{black}Variation trends of performance with varying edge densities.}
        \label{fig:vted}
\end{figure*}

\subsection{Generality on Testing Graph Sizes}
We offer an analytical perspective on the performance of the SiCL model when applied to larger WS-L-G graphs. 
It is important to highlight that the models were initially trained on graphs comprising 30 vertices, positioning this task within an out-of-distribution setting in terms of graph size. 
To establish a point of reference, we have included results from the PC algorithm as a baseline comparison.
These findings can be examined in Tab. \ref{tab:mpva}.
Despite the OOD conditions, SiCL maintains robust performance, reinforcing its scalability and the model's general applicability across varying graph sizes.

\begin{table}[t]\color{black}
\centering
\begin{threeparttable}
\caption{\color{black}Performance comparison with varying amounts of graph sizes.}
\label{tab:mpva}
\begin{tabular}{l|ccc|ccc|ccc}
\toprule
Metric & \multicolumn{3}{c|}{s-F1$\uparrow$} & \multicolumn{3}{c|}{v-F1$\uparrow$} & \multicolumn{3}{c}{o-F1$\uparrow$} \\
Size & 50 & 70 & 100 & 50 & 70 & 100 & 50 & 70 & 100 \\
\midrule
PC       & $17.7$ & $14.8$ & $10.6$ & $6.4$ & $5.0$ & $3.7$ & $7.0$ & $5.6$ & $4.0$ \\
SiCL & $\mathbf{41.6}$ & $\mathbf{37.4}$ & $\mathbf{28.3}$ & $\mathbf{34.9}$ & $\mathbf{30.7}$ & $\mathbf{22.6}$ & $\mathbf{37.9}$ & $\mathbf{33.7}$ & $\mathbf{24.8}$ \\
\bottomrule
\end{tabular}
\end{threeparttable}
\end{table}

\subsection{Acyclicity}
\begin{table}[t]\color{black}
\centering
\begin{threeparttable}
\caption{\color{black}Count of cycles in the CPDAG predictions without post-processing of removing cycles.}
\label{tab:ccfc}
\begin{tabular}{@{}ccc@{}}
\toprule
Dataset & WS-L-G & SBM-L-G  \\
\midrule
Rate of Graphs with Cycles & $0.66 \pm 0.66 \%$&$0.00 \pm 0.00 \%$ \\
\bottomrule
\end{tabular}
\end{threeparttable}

\end{table}
We provide an empirical evidence supporting of the rarity of cycles in the predictions. The experimental data presented in Tab. \ref{tab:ccfc} corroborates that cycles are infrequently observed in the predicted CPDAGs, even though without any post-processing on removing cycles.
\color{black}

\end{document}